%% file: main.tex
  \documentclass[sigconf]{acmart}
\renewcommand\footnotetextcopyrightpermission[1]{} 
\usepackage[table,xcdraw]{xcolor}
\usepackage[utf8]{inputenc} 
\usepackage[T1]{fontenc}    
\usepackage{hyperref}       
\usepackage{amsfonts}       
\usepackage{microtype}      

\usepackage{amsmath,amsthm,amssymb}
\usepackage{bm}
\usepackage{siunitx}
\usepackage{nicefrac}

\usepackage{array,booktabs,threeparttable}
\usepackage{multirow,makecell}
\usepackage{tabularx}
\usepackage{stfloats}
\usepackage{balance}
\usepackage{caption}

\usepackage{graphicx}
\usepackage{epstopdf}
\usepackage[procnumbered,ruled,vlined,linesnumbered]{algorithm2e}

\usepackage{enumerate}

\usepackage{enumitem}

\usepackage{tcolorbox}

\input{symbol_def}

\author{Haoxin Sun}
\affiliation{%
\institution{Fudan University}
\city{Shanghai}
\country{China}}
\email{23110240089@m.fudan.edu.cn}

\author{Yiqing Lin}
\affiliation{%
\institution{ByteDance}
\city{Shanghai}
\country{China}}
\email{linyiqing@bytedance.com}

\author{Yajun Huang}
\affiliation{%
\institution{ByteDance}
\city{Shanghai}
\country{China}}
\email{huangyajun.gooday@bytedance.com}

\author{Chenhui Dong}
\affiliation{%
\institution{ByteDance}
\city{Shanghai}
\country{China}}
\email{dongchenhui.0711@bytedance.com}

\author{Mingjun Li}
\affiliation{%
\institution{ByteDance}
\city{Shanghai}
\country{China}}
\email{limingjun.tsinghua@bytedance.com}

\author{Zhongzhi Zhang}
\affiliation{%
\institution{Fudan University}
\city{Shanghai}
\country{China}}
\email{zhangzz@fudan.edu.cn}

\keywords{Graph Pretrain, Large Language Model, Graph Neural Network, Graph-Level Tasks}
\begin{document}

\title{GLIP: Graph and LLM Joint  Pretraining for Graph-Level Tasks }

\begin{abstract}
 
Graphs are widely used to model relational systems, with applications in domains such as social networks, finance, and biomedicine. Graph neural networks (GNNs) have become a mainstream approach for learning graph representations. With the rise of large language models (LLMs), recent studies have attempted to combine GNNs with LLMs. However, most existing works concentrate on node-level and edge-level tasks, while graph-level tasks, which require capturing more complex structural and feature information, remain relatively underexplored. Moreover, graph pretraining is a widely adopted strategy to alleviate the challenge of label scarcity. Most existing approaches are designed solely for GNNs such as GraphCL, leaving LLMs uninvolved in the process. To address these limitations, 
we propose GLIP, a Graph–LLM JoInt Pretraining framework for graph-level tasks. GLIP first performs graph augmentation to construct positive and negative pairs and introduces a multi-token selection strategy to identify patches informative in both structure and features.
It further leverages a diffusion-based projector to enrich them with contextual information, enabling GLIP to capture signals from both global and local perspectives. 
Finally, GLIP employs a joint objective that integrates the LLM’s semantic judgments with a contrastive alignment loss, ensuring consistent supervision at both the semantic and structural levels.After pretraining, GLIP is fine-tuned with limited labeled data for downstream tasks, and extensive experiments show that it outperforms state-of-the-art methods on graph-level classification and reasoning tasks.  Our source code is publicly available on \url{https://anonymous.4open.science/r/GLIP}. 

\end{abstract}
\maketitle

\section{Introduction}

Graphs are a fundamental data structure for modeling relational systems, with broad applications in diverse domains such as social networks~\cite{perozzi2014deepwalk,grover2016node2vec}, financial systems~\cite{dou2020enhancing,ma2021comprehensive}, and biomedical sciences~\cite{gilmer2017neural,stokes2020deep}.
Graph neural networks (GNNs)~\cite{kipf2016semi,velivckovic2017graph,hamilton2017inductive}, which leverage message passing and neighborhood aggregation mechanisms, have emerged as powerful frameworks for graph representation learning, achieving strong performance in downstream tasks such as node classification, graph classification, and link prediction~\cite{wu2020comprehensive,zhang2020deep}. While GNNs have achieved remarkable progress in graph-structured data representation, the rise of large language models (LLMs) has reshaped modern artificial intelligence, demonstrating impressive reasoning and generalization abilities across diverse domains~\cite{brown2020language,touvron2023llama,bai2023qwen,guo2025deepseek}.  In recent years, an increasing number of studies have attempted to integrate GNNs with LLMs, combining the structural encoding strengths of GNNs with the inference capacity of LLMs, leading to complementary effects that surpass either model alone~\cite{chai2023graphllm,chen2024llaga,tang2024graphgpt,zhang2024graphtranslator,wang2024llms}.

Although recent efforts have explored the integration of GNNs with LLMs, existing approaches still face several limitations.
In particular, prior studies have mainly concentrated on node-level or link-level tasks, while graph-level tasks such as classification and explanation remain relatively underexplored. However, these tasks are critical in many domains, such as detecting fraudulent groups with explanations in financial networks~\cite{dou2020enhancing,rossi2018deep}, as well as identifying toxic molecules and uncovering their underlying mechanisms in biomedical applications~\cite{gilmer2017neural,stokes2020deep}. In addition, most existing methods that combine graph encoders with LLMs rely heavily on supervised fine-tuning with abundant labeled data, which is costly and often impractical. While several label-free pretraining methods have been proposed~\cite{you2020graph,hou2022graphmae,velivckovic2018deep}, they are designed for GNN alone and do not involve LLM in the pretraining process. These limitations naturally raise a key question.  \textit{How can we design a joint pretraining strategy for graph encoders and LLMs that enables graph representations to be effectively utilized by LLMs for graph-level tasks under limited labeled data?}

Two challenges arise in establishing a joint pretraining strategy:\\
\textbf{(i) How can LLMs be involved in pretraining to provide semantic-level supervision for graph tokens?} Existing pretraining methods for graphs~\cite{you2020graph,hou2022graphmae,velivckovic2018deep} are designed exclusively for graph encoders. Although recent work~\cite{wang2024llms} leverages the principal components of LLM token embeddings to align node representations, the LLM’s reasoning and response capabilities are not leveraged in the pretraining process.\\
\textbf{(ii) How can we present graph information to LLMs at both global and local levels?} Although prior works have explored aligning graphs with LLMs~\cite{chen2024llaga,tang2024graphgpt,zhang2024graphtranslator,wang2024llms}, they primarily focus on node-level tasks. However, graph-level tasks require both global structural context and critical local substructures. Simple pooling or linear projection often loses such information, making it difficult for LLMs to reason effectively over graphs.

In this paper, we propose \textbf{GLIP} (\textbf{G}raph–\textbf{L}LM Jo\textbf{I}nt \textbf{P}retraining), a novel pretraining framework that enables graph encoders and LLMs to collaboratively participate in the pretraining process through a shared training objective. GLIP introduces a multi-token selection strategy that captures both global structural context and critical local substructures. Furthermore, it incorporates a diffusion-based projector that builds contextual connections among different local tokens, thereby enabling graph tokens to better elicit reasoning behaviors from LLMs. Moreover, GLIP employs a label-free, self-supervised pretraining objective that explicitly involves LLMs in the pretraining process, allowing the pretraining process, guided by the LLM’s semantic judgments, to organize graph representations according to their similarities and differences across diverse graph structures. Our contributions are summarized as follows:

 \begin{itemize}[leftmargin=*]
    \item We propose GLIP, the first joint pretraining framework that enables graph encoders and LLMs to collaboratively participate in the pretraining process through a shared training objective, and supports efficient fine-tuning with limited labeled data for downstream tasks.
    \item  We design a multi-token selection strategy to capture both global context and local substructures, and propose a diffusion-based projector that enriches contextual dependencies among local information.
    \item  We conduct comprehensive experiments on multiple datasets covering diverse domains and graph types, and demonstrate that GLIP consistently outperforms state-of-the-art baselines in graph-level classification and explanation tasks.
\end{itemize}

\begin{figure*}[t]
  \centering
  \includegraphics[width=.95\linewidth]{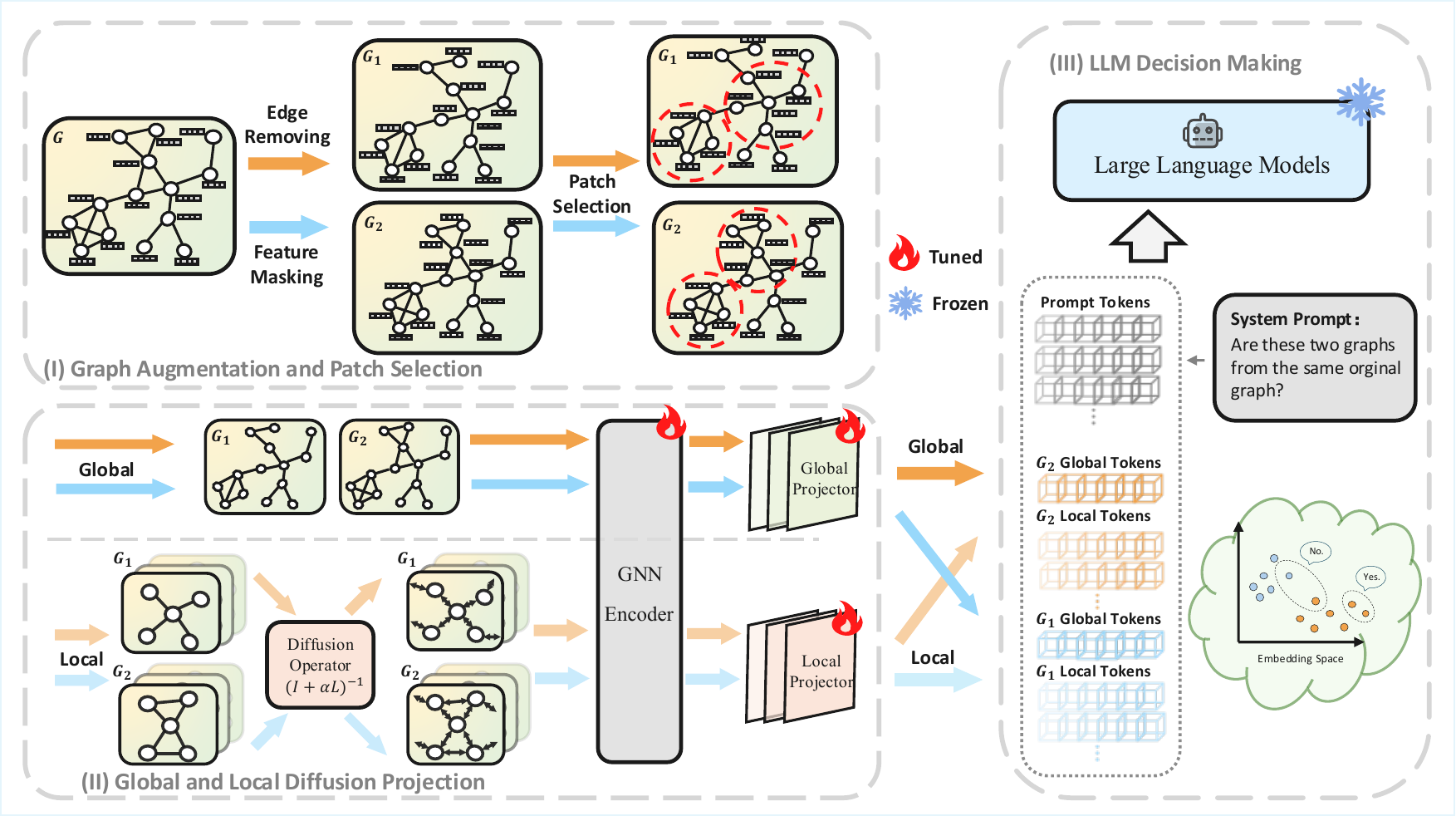}
  \vspace{-3mm}
  \caption{The overall framework of Glip.}
  \vspace{-3mm}
  \label{fig:overview}
\end{figure*}
  \vspace{-2mm}
\section{Related Work}
\noindent\textbf{ Self-supervised Pretraining Method for GNNs.}
Graph Neural Networks (GNNs) have become a central paradigm in graph machine learning, attracting significant attention in recent years~\cite{wu2020comprehensive,wu2022graph,zhou2020graph,xu2018powerful,velivckovic2023everything}. The core idea behind GNNs is to learn expressive graph representations through message passing and neighborhood aggregation. Several representative architectures have been proposed, including Graph Attention Networks (GAT)~\cite{velivckovic2017graph}, Graph Convolution Network (GCN)~\cite{zhang2020gcn},
and GraphSAGE~\cite{hamilton2017inductive}.

To reduce the reliance on labeled data and improve the generalization ability of GNNs, self-supervised learning has been widely explored for pretraining. Deep Graph Infomax (DGI)~\cite{velivckovic2018deep} maximizes mutual information between local and global representations. Contrastive learning methods such as GraphCL~\cite{you2020graph} construct positive and negative pairs through graph augmentations to enhance discriminability. More recently, SimGRACE~\cite{xia2022simgrace} achieves contrastive learning without explicit graph augmentations by leveraging dropout-induced perturbations within the encoder, while generative approaches like GraphMAE~\cite{hou2022graphmae} learn representations through reconstructing masked features or graph structures.

\noindent\textbf{Large Language Models for Graphs.} 
With the rapid advancement of Large Language Models (LLMs) and their strong generalization capabilities, leveraging LLMs to tackle transferability issues in graph machine learning has attracted growing interest~\cite{wu2025llms,guo2023gpt4graph,jin2024large,li2023survey,chen2024exploring}. Some studies convert graph structures into textual inputs for LLMs~\cite{chen2024exploring,wang2023can}, but such representations often yield suboptimal performance~\cite{huang2023can}. LLMs have also been adopted as enhancers~\cite{chen2023label,liu2023one}, typically generating synthetic data or node textual representations.  More recently, growing efforts have focused on employing LLMs directly as predictors. For instance,  LLaGA~\cite{chen2024llaga} presents an encoding method that translates graph data into LLM compatible sequences, while GraphGPT~\cite{tang2024graphgpt} integrates LLMs with graph structures via instruction tuning, enabling strong zero-shot generalization across graph tasks.

However, most of these methods rely on labeled data for supervised learning, which limits their applicability when labels are scarce. Some recent works have explored pretraining approaches such as contrastive learning for combining graphs and LLMs.  For example in~\cite{wang2024llms}, authors propose training a GNN to align its entity representations with the token embeddings of an LLM. Nevertheless, the LLM itself is not involved during the pretraining process. These limitations motivate us to design a joint pretraining framework that explicitly involves both GNNs and LLMs to capture graph semantics in a label-free manner.

\section{Methodology}

\subsection{Notation} 
Let $\mathcal{G} = \{\mathcal{V}, \mathcal{E}, \boldsymbol{X}\}$ denote a connected undirected graph, where $\mathcal{V} = \{v_1, v_2,...,v_N\}$ is the set of $N$ nodes, $\mathcal{E}=\{e_{ij}\}$ is the set of edges, and $\boldsymbol{X} \in \mathbb{R}^{n \times F}$ is node features. In what follows, $v_i$ and $i$ are used interchangeably to represent node $v_i$ if incurring no confusion. Let $\boldsymbol{A}$ be the adjacency matrix, and $\boldsymbol{D}$ be the diagonal degree matrix with entries $d_i = \boldsymbol{D}_{ii} = \sum_j \boldsymbol{A}_{ij}$. The graph Laplacian is defined as $\boldsymbol{L} = \boldsymbol{D} - \boldsymbol{A}$. For any node $i$, we denote its neighborhood as $N_i = \{ j : e_{ij} \in \mathcal{E} \}$.

  \vspace{-2mm}
\subsection{Overall Architecture}



The central goal of GLIP is to learn graph representations aligned with the embedding space of a pretrained LLM in a self-supervised manner, by jointly training a graph encoder and a language-model interface. Instead of updating the LLM itself, GLIP keeps the LLM frozen and uses it as a semantic decision interface, while optimizing a GNN-based encoder and projection modules that transform graphs into token sequences compatible with the LLM input space. In this setting, the LLM provides structured supervision, and the graph side is trained to produce representations that are well separated in the LLM embedding space.

GLIP follows a two-stage pipeline: (i) a self-supervised pretraining stage that aligns graph tokens in the LLM token space via a view-consistency objective, and (ii) a task-specific fine-tuning stage that adapts the pretrained graph encoder and projectors with limited labeled data for downstream tasks such as graph classification and reasoning. The pretraining stage reshapes the representation space by pulling semantically similar graphs closer and pushing dissimilar ones apart, which enables more sample-efficient learning of decision boundaries during fine-tuning.

As illustrated in Figure~\ref{fig:overview}, GLIP consists of three key components in the pretraining stage:
(1) \textbf{Graph augmentation and patch selection}. Given a graph $\mathcal{G}$, we apply augmentation operations such as edge removal or node feature masking to generate two correlated views, $\mathcal{G}_1$ and $\mathcal{G}_2$, which form a positive pair, while different original graphs are treated as negative pairs. For each graph, we design a patch selection strategy to extract several informative subgraphs. The entire graph is regarded as the global view, while each patch is considered as a local view, enabling the model to capture both global and local structural information.
(2) \textbf{Global projection and local diffusion projection}. For the global view, features are passed through a GNN-based encoder followed by a projector into the LLM embedding space. For the local views, we introduce a diffusion operator that enriches each subgraph with broader structural context before feeding them into the same GNN encoder and projector.
(3) \textbf{LLM-based semantic discrimination}. The projected graph tokens, together with prompt tokens, are fed into a frozen open-source LLM, which is trained to predict whether two token sequences originate from the same graph. This binary decision objective provides a semantic-level supervision signal that encourages graph tokens from the same graph to be close, while pushing apart those from different graphs in the LLM token space.

With this pretraining procedure, GLIP learns graph encoders and projectors that are well aligned with the LLM token space without requiring any task labels. In the subsequent fine-tuning stage, we initialize the model with the pretrained parameters and adapt the graph encoder and projectors using a small amount of labeled data, enabling the model to form effective decision boundaries for downstream graph-level tasks such as semi-supervised classification, few-shot learning, and reasoning.

 \subsection{  Patch Selection}
As discussed in the previous section, GLIP represents each graph using a global token and multiple local tokens to capture informative substructures. In this section, we introduce how to select local tokens under a limited token budget.

\subsubsection{Patch Selection Candidates}
Given a graph $\mathcal{G}$ with many nodes and edges, a straightforward approach is to treat the entire graph as a global view, obtaining a single graph token by pooling over all nodes.  
However, relying solely on the global view may overlook important local patterns that are crucial for understanding the graph. To address this, we introduce a patch selection framework. Specifically, we assume a collection of $t$ candidate patches (subgraphs) $\mathcal{C} = \{S_1, S_2, \dots, S_t\}$, such as one-hop neighborhoods around nodes.  

\subsubsection{Patch Selection Problem}
Each candidate patch contains meaningful local information. However, we cannot select all candidates due to computational cost, memory constraints, and the token budget for LLMs. Suppose the budget is $k$. Our goal is to select $k$ informative patches  that effectively capture both the structural and feature-level characteristics of the graph. Formally, the patch selection problem is defined as follows.

\begin{tcolorbox}[boxrule={1pt}]
\begin{problem}\label{Pr-opt}[Patch Selection Problem]
    Given a collection of $t$ candidate patches $\mathcal{C} = \{S_1, S_2, \dots, S_t\}$ from the graph $\mathcal{G}$, the goal is to select $k$ patches by solving:
    \[
    \begin{aligned}
    \max_{U \subseteq \mathcal{C}} \quad &\lambda f_{\mathrm{structure}}(U) +(1-\lambda) f_{\mathrm{feature}}(U) \\
    \text{s.t.} \quad &|U| = k.
    \end{aligned}
    \]
\end{problem}
\end{tcolorbox}

\begin{figure*}[t]
  \centering
  \includegraphics[width=.7\linewidth]{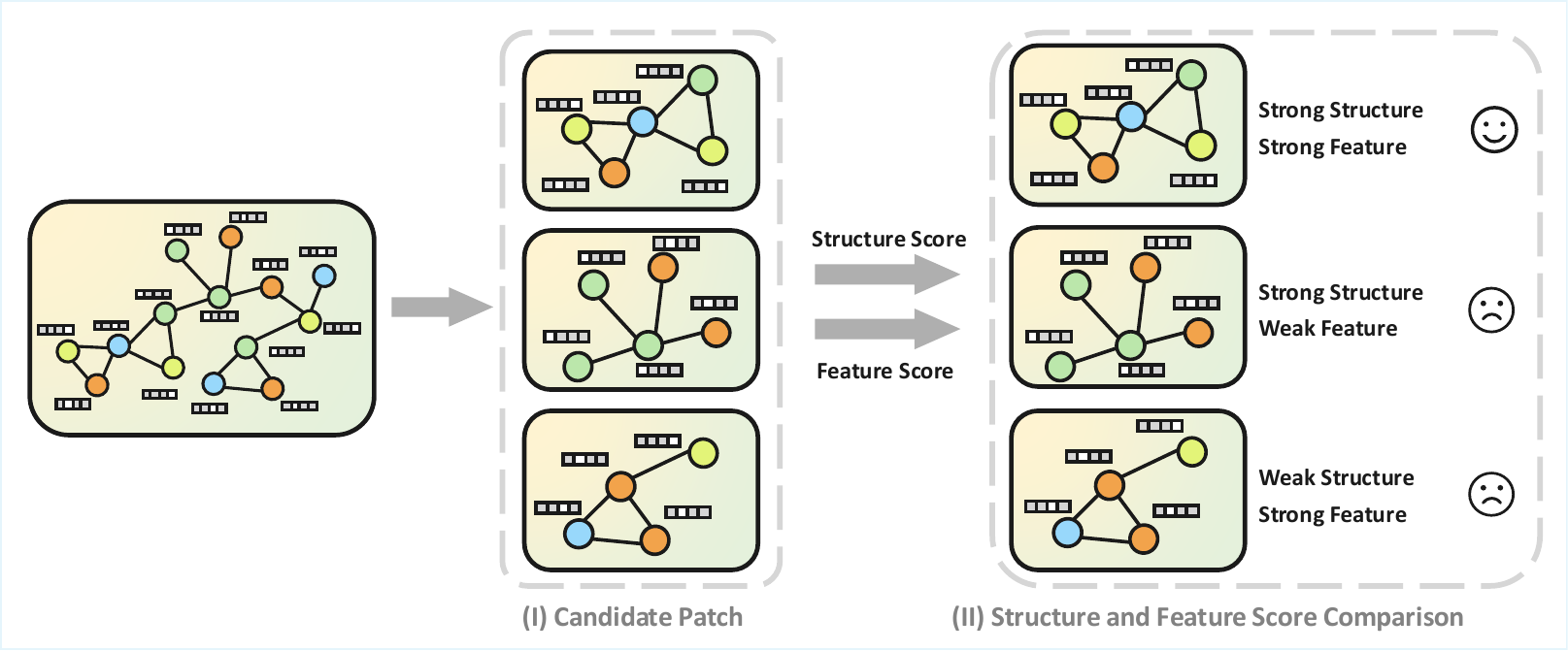}
  \vspace{-3mm}
  \caption{The  framework of patch selection.}
  \vspace{-3mm}
  \label{fig:patch-selection}
\end{figure*}

We use $f_{\text{structure}}(U)$ and $f_{\text{feature}}(U)$ to capture structural and feature-level objectives respectively, weighted by parameter $\lambda \in [0,1]$. The function $f_{\text{structure}}(U)$ measures structural coverage, encouraging the union  of selected patches to maximize the number of distinct nodes.  {Let $C_U = \bigcup_{S \in U} S$ denote the set of nodes covered by the selected patches. We define the structural coverage as $f_{\text{structure}}(U) = \frac{|C_U|}{N}$, where $N$ is the total number of nodes in the graph.}


Besides structural coverage, we also expect the union set $U$ to capture diverse feature information of nodes across the graph. To this end, we define $f_{\text{feature}}(U)$ to measure feature-level coverage. Each node $v$ is associated with a feature embedding vector  $x_v$. For each node $v$ in the graph, we compute its maximum cosine similarity to any node $u$ within the selected union $U$. The overall feature coverage is then given by the average of these maximum similarities across all nodes $f_{\text{feature}}(U) = \frac{1}{N} \sum_{v=1}^{N} \max_{u \in C_U} \cos(\xx_v, \xx_u),
$, where $x_v$ and $x_u$ denote the normalized feature embeddings of nodes $v$ and $u$, respectively.

Define  $F(U)$ as    $F(U) = \lambda f_{\mathrm{structure}}(U) + (1-\lambda) f_{\mathrm{feature}}(U).$ Then, Problem~\ref{Pr-opt} aims to find a set $U$ that maximizes $F(U)$. 

\subsubsection{Hardness and Properties of Problem~\ref{Pr-opt}}

Since we need to select $k$ patches out of $t$ candidates, Problem~\ref{Pr-opt} is inherently combinatorial. In fact, we now show that this problem is NP-hard.

\begin{theorem}\label{th-NP}
The Patch Selection Problem (Problem~\ref{Pr-opt}) is NP-hard.
\end{theorem}



The proof of Theorem~\ref{th-NP} is provided in the Appendix. Since the Patch Selection Problem is NP-hard, obtaining an exact solution is computationally infeasible for large graphs. Nevertheless, the objective function $F(U)$ is monotone and submodular with respect to the selected set $U$, as shown in the following theorem. These properties enable the use of a greedy algorithm with provable approximation guarantees.

\begin{theorem}\label{th-mo_and_sub}
Let $\mathcal{C}=\{S_1,S_2,\dots,S_t\}$ be the candidate patch family. 
Define the objective function $F(U) = \lambda f_{\mathrm{structure}}(U) + (1-\lambda) f_{\mathrm{feature}}(U)$ for any $U \subseteq \mathcal{C}$. 
Then the objective function is monotone and submodular, i.e., $F(A \cup \{S_i\}) - F(A) \ge 0$ for all $A \subseteq \mathcal{C}$ and $S_i \in \mathcal{C} \setminus A$, and 
$F(A \cup \{S_i\}) - F(A) \ge F(B \cup \{S_i\}) - F(B)$ for all $A \subseteq B \subseteq \mathcal{C}$ and $S_i \in \mathcal{C} \setminus B$.
\end{theorem}

The detailed proof of  Theorem~\ref{th-mo_and_sub} is provided in the Appendix.

\subsubsection{Greedy Algorithm}

By Theorem~\ref{th-NP}, the Patch Selection Problem is NP-hard. Nevertheless, Theorem~\ref{th-mo_and_sub} shows that the objective function $F(U)$ is monotone and submodular. According to the classical result of Nemhauser et al.~\cite{nemhauser1978analysis}, a greedy strategy achieves a $(1-1/e)$-approximation to the optimal solution. The pseudocode is shown in Algorithm~\ref{alg:greedy} in the appendix. The algorithm proceeds for $k$ iterations, and in each round greedily selects the patch from the remaining candidates that provides the largest marginal gain of $F(U)$, which balances structural and feature coverage.











Suppose there are $t$ candidate patches, each of size at most $p$, the graph has $N$ nodes with feature dimension $F$, and we aim to select $k$ patches. Precomputing patch–node similarities costs $O(tNpF)$. Each of the $k$ iterations then scans all remaining candidates with $O(N)$ cost, yielding an additional $O(ktN)$. With precomputation and incremental updates, the overall time complexity of the greedy algorithm is $O(t N F p + k t N)$, which scales efficiently in practice while still providing provable approximation guarantees.

\subsection{Global Projection and Local Diffusion Projection}
In this subsection, we describe how graph information is mapped into the LLM representation space. To preserve both global and local structures, we design two complementary modules: a global projection path, which encodes the entire graph via a GNN and a projector, and a local diffusion projection path, which enriches selected patches with contextual information before projection.

\subsubsection{Diffusion Operator}
For the local view, encoding each patch separately may lose important context, because nodes inside a patch have limited awareness of the overall graph structure. To address this, we introduce an information diffusion process over the graph. Each node $i$ is initialized with its feature vector $\zz_i(0)=\xx_i \in \mathbb{R}^{F\times 1}$. At each diffusion step $q=1,2,\dots$, the feature of node $i$ is updated by aggregating its own feature and those of its neighbors:
\begin{equation}\label{FJeq}
\zz_i(q+1) = \frac{\xx_i + \alpha \sum_{j \in N_i} \zz_j(q)}{1 + \alpha d_i},
\end{equation}
where $d_i$ is the degree of node $i$ and $\alpha$ is a diffusion coefficient. Suppose that 
$\XX=[\xx_1,\dots,\xx_N]^\top$ denotes the matrix of input features, 
and $\ZZ(q)=[\zz_1(q),\dots,\zz_N(q)]^\top$ stacks the node embeddings after $q$ diffusion steps.  The following theorem shows that as $q \to \infty$, $\ZZ(q)$ converges to $(\II+\alpha \LL)^{-1}\XX$.

\begin{theorem}\label{th-diffusion}
For any $\alpha>0$ and any initialization $\ZZ(0)$, the iteration \eqref{FJeq} converges to a unique fixed point $\ZZ^\star$ that satisfies $\ZZ^\star \;=\; (\II+\alpha \LL)^{-1}\XX$.
\end{theorem}



The proof of Theorem~\ref{th-diffusion} is included in the Appendix. Specifically, Theorem~\ref{th-diffusion} shows that $(\II+\alpha \LL)^{-1}$ serves as the fundamental matrix in this diffusion process. This matrix can be interpreted both as a kernel function and as a low-pass filter, allowing local features to be blended with long-range structural information. Compared with conventional GNN message passing, this diffusion operator provides a more direct way to incorporate long-range dependencies. As a result, it enriches local patch representations with global context, thereby improving the quality of graph representations presented to the LLM.



\subsubsection{GNN Encoder and Projector}
After preparing both global and local graph views, we employ a GNN encoder together with two projection modules to transform node features into latent representations.

\textbf{Global Path.} For the global view, the entire graph $\mathcal{G}$ is passed into a GNN encoder, producing node-level embeddings. We then apply a pooling operation over all nodes to derive a compact global representation, which is projected into the LLM token space through a lightweight multi-layer perceptron (MLP), serving as the global projector.

\textbf{Local Path.} For the local view, each selected patch is first processed by the diffusion operator described in the previous subsection, yielding diffused node embeddings. These embeddings are then encoded by a GNN encoder, followed by pooling within each patch to obtain patch-level representations. Each patch representation is subsequently mapped into the LLM token space through another MLP, referred to as the local projector.

This dual-path design allows the global projector to capture an overall view of the graph, while the local projector delivers enriched, fine-grained subgraph representations. Together, they generate a multi-token sequence that integrates both global context and local structural details, providing the LLM with graph tokens that carry richer contextual information for reasoning.

\subsection{Loss Design}

We adopt an open-source LLM (Mistral-7B and Qwen2.5-7B) as the backbone. In our framework, the LLM is used as a semantic decision interface, while the graph encoder and projection modules are optimized to produce graph tokens that are well aligned with the LLM token space. We design different training objectives for the pretraining and fine-tuning stages, which we describe below.

\subsubsection{Pretraining Stage}

In the pretraining stage, we aim to align graph representations in the LLM token space in a self-supervised manner. Suppose a batch contains $B$ graphs, and each graph $G$ is augmented into two correlated views, denoted as $\mathcal{G}_1$ and $\mathcal{G}_2$. Each pair $(\mathcal{G}_1,\mathcal{G}_2)$ constitutes a positive sample, while pairs formed from different original graphs are treated as negative samples. Based on this setup, our pretraining objective consists of two components: an LLM-based prediction loss and a contrastive alignment loss.

\paragraph{LLM Prediction Loss.}
To leverage the LLM as a semantic decision interface in graph representation learning, we adopt a prompt-based prediction task. Specifically, we prepend a natural language prompt such as \textit{``Are these two graphs from the same original graph?''} to the input sequence composed of graph tokens.

Supervision is applied using a cross-entropy objective. Pairs of two augmented views from the same graph are treated as positive samples, with the target label \textit{``Yes''}. Negative samples are formed by pairing the first augmentation of one graph with the second augmentation of a different graph. To maintain balance, the number of negative pairs is kept equal to that of positive pairs. For these negative pairs, the target label is \textit{``No''}. The resulting cross-entropy loss is denoted as $\mathcal{L}_{\text{LLM}}$.

\paragraph{Contrastive Alignment Loss.}
To further enforce consistency between different augmented views, we introduce a contrastive alignment loss based on the InfoNCE objective. For each graph, two augmented views are generated, and each view is mapped into one global token and multiple local tokens. We compare both the global tokens and the averaged local tokens across views.

Specifically, let $\uu_i^{(1)}$ and $\uu_j^{(2)}$ denote the embeddings of the $i$-th and $j$-th graphs from the two views (either global or averaged local). The InfoNCE loss is defined as
\begin{equation}
\mathcal{L}_{\text{NCE}} 
= -\frac{1}{B}\sum_{i=1}^B 
\log \frac{\exp\big(s(\uu_i^{(1)},\uu_i^{(2)})/\tau\big)}
{\sum_{j=1}^B \exp\big(s(\uu_i^{(1)},\uu_j^{(2)})/\tau\big)},
\end{equation}
where $s(\cdot,\cdot)$ denotes cosine similarity and $\tau$ is a temperature hyperparameter.

We combine the global- and local-level objectives using a convex weight $\beta \in [0,1]$: $\mathcal{L}_{\text{contrast}} 
= \beta \,\mathcal{L}_{\text{NCE\_Global}} 
+ (1-\beta)\,\mathcal{L}_{\text{NCE\_Local}}$.

\paragraph{Final Pretraining Loss.}
The overall pretraining loss combines the LLM prediction objective with the contrastive alignment loss: $\mathcal{L}_{\text{pre}} = \gamma \cdot \mathcal{L}_{\text{LLM}} + (1-\gamma) \cdot \mathcal{L}_{\text{contrast}}$, where $\gamma \in [0,1]$ balances the two components.

This pretraining objective uses the LLM as a semantic decision interface to guide the graph encoder and projectors to organize graph tokens according to their similarities and differences in the LLM token space, without requiring any task-specific labels.

\subsubsection{Fine-tuning Stage}
In the fine-tuning stage, we adapt the pretrained model to downstream tasks using labeled data. Depending on the task, the supervision signal can be either a discrete class label (for graph classification) or a natural language output (for graph reasoning). We formulate fine-tuning as a conditional prediction problem and optimize a standard cross-entropy loss over the target outputs. Concretely, the graph encoder and projection modules are initialized with the parameters learned in the self-supervised pretraining stage and are further updated to fit task-specific objectives. This design transfers the semantic alignment established during pretraining to task-specific decision boundaries, enabling more sample-efficient learning under limited labeled data.

\section{Experiments}
\label{exp}

In this section, we present experiments to systematically evaluate the effectiveness of GLIP.   We aim to answer the following research questions:  
\textbf{RQ1:} How well does GLIP perform on semi-supervised graph classification?  
\textbf{RQ2:} How well does GLIP perform on few-shot graph classification?  
\textbf{RQ3:} How effective is GLIP for graph reasoning tasks? 
\textbf{RQ4:} What is the impact of each component in the modular design of GLIP?  

\begin{table}[t!]
\centering
\fontsize{8.5}{12}\selectfont  
\setlength{\tabcolsep}{3pt} 
\caption{Statistics of the datasets used in experiments.}
\label{tb-dataset}
\begin{tabular}{ccccccc}
\hline
Dataset      & \#Graphs & \#Nodes & \#Edges & \#Feats & Neg:Pos & TAG   \\ \hline
MUTAG         & 188           & 17.93     & 39.59     & 7               & 63:125    & No         \\
PROTEINS      & 1113          & 39.06     & 145.63    & 4               & 663:450   & No         \\
REDDIT & 2000          & 429.63    & 995.51    & 0               & 1000:1000 & No         \\
IMDB   & 1000          & 19.77     & 193.06    & 0               & 500:500   & No         \\ \hline
BBBP          & 2039          & 24.06     & 51.91     & 9               & 479:1560  & Yes        \\
BACE          & 1513          & 34.09     & 73.72     & 9               & 822:691   & Yes        \\
E-Com         & 2000           & 18.61     & 425.84    & 526             & 1000:1000   & Yes        \\ \hline
\end{tabular}
\end{table}

\subsection{Experimental Setup}

\noindent\textbf{Datasets.} 
We evaluate GLIP on seven datasets that cover both feature-based and text-attributed graphs (TAG), as summarized in Table~\ref{tb-dataset}. The feature-based datasets include MUTAG, PROTEINS, REDDIT-BINARY (abbreviated as REDDIT), and IMDB-BINARY (IMDB), all from the TUDataset collection~\cite{morris2020tudataset}. MUTAG and PROTEINS are molecular and protein datasets with 7- and 4-dimensional node features, respectively. REDDIT and IMDB lack intrinsic node features, so we follow standard practice by using node degrees as feature inputs. The text-attributed datasets include BBBP, BACE, and E-Commerce (E-Com).
BBBP and BACE are molecular property prediction benchmarks from~\cite{feng2024taglas}, where each node is accompanied by textual annotations of its chemical properties.
E-Com is an industrial e-commerce dataset, where each graph represents a community of TikTok Shop sellers, and the task is to predict whether the group is risky. Table~\ref{tb-dataset} reports dataset statistics, where \#Graphs denotes the number of graphs, \#Nodes and \#Edges represent the average number of nodes and edges per graph, \#Feats is the node feature dimension, Neg:Pos indicates the negative/positive sample ratio, and TAG shows whether the dataset contains text-attributed graphs.

\begin{table*}[htbp!]
\centering
\caption{\textbf{Performance comparison under semi-supervised settings  with Accuracy (\%) reported.} 
The \colorbox[HTML]{FFDEC7}{\textbf{best}} and \colorbox[HTML]{FFF2E9}{second-best} results are highlighted. Results are averaged over five runs. }
\label{tb-semi-accuracy}
\fontsize{8.5}{13}\selectfont  
\setlength{\tabcolsep}{4pt} 
\begin{tabular}{cc|ccccc|cccc}
\hline
\multicolumn{2}{c|}{} &
  \multicolumn{5}{c|}{\textbf{Feature Graphs}} &
  \multicolumn{4}{c}{\textbf{Text-Attributed Graphs}} \\
\multicolumn{2}{c|}{\multirow{-2}{*}{\textbf{\begin{tabular}[c]{@{}c@{}}Accuracy \\Semi-Supervised  \end{tabular}}}} &
  \textbf{MUTAG} &
  \textbf{PROTEINS} &
  \textbf{REDDIT} &
  \textbf{IMDB} &
  \textbf{Average} &
  \textbf{BBBP} &
  \textbf{BACE} &
  \textbf{E-Com} &
  \textbf{Average} \\ \hline
 &
  GraphMAE &
  $66.89_{\pm 1.32}$ &
  $60.23_{\pm 1.87}$ &
  $80.40_{\pm 4.30}$ &
  $53.08_{\pm 5.78}$ &
  $65.15$ &
  $76.52_{\pm 0.02}$ &
  $54.11_{\pm 0.46}$ &
  $83.38_{\pm 1.24}$ &
  $71.34$ \\
 &
  GraphCL &
  $70.84_{\pm 5.04}$ &
  $67.48_{\pm 1.96}$ &
  $86.19_{\pm 2.02}$ &
  $60.85_{\pm 2.47}$ &
  $71.34$ &
  $77.22_{\pm 1.28}$ &
  $56.48_{\pm 2.70}$ &
  $83.14_{\pm 2.03}$ &
  $72.28$ \\
 &
  DGI &
  $68.75_{\pm 3.09}$ &
  $61.66_{\pm 2.53}$ &
  $84.92_{\pm 2.20}$ &
  $55.75_{\pm 2.72}$ &
  $67.77$ &
  $76.53_{\pm 11.20}$ &
  $56.32_{\pm 2.79}$ &
  $82.11_{\pm 2.04}$ &
  $71.65$ \\
\multirow{-4}{*}{\textbf{\begin{tabular}[c]{@{}c@{}}Graph\\ Pretrain\\ Method\end{tabular}}} &
  SimGRACE &
  $67.95_{\pm 2.35}$ &
  $60.87_{\pm 2.20}$ &
  $83.35_{\pm 2.34}$ &
  $60.15_{\pm 3.08}$ &
  $68.08$ &
  $76.41_{\pm 0.24}$ &
  $54.07_{\pm 4.03}$ &
  $83.18_{\pm 1.28}$ &
  $71.22$ \\ \hline
 &
  GraphGPT-Mistral &
  - &
  - &
  - &
  - &
  - &
  $78.42_{\pm 2.32}$ &
  $54.76_{\pm 1.34}$ &
  $80.91_{\pm 1.01}$ &
  $71.36$ \\
 &
  GraphGPT-Qwen &
  - &
  - &
  - &
  - &
  - &
  $77.22_{\pm 1.34}$ &
  $55.76_{\pm 1.27}$ &
  $81.33_{\pm 2.13}$ &
  $71.44$ \\
 &
  LLaGA-Mistral &
  $68.35_{\pm 4.38}$ &
  $59.46_{\pm 0.30}$ &
  $75.84_{\pm 3.55}$ &
  $52.20_{\pm 2.67}$ &
  $63.96$ &
  $76.53_{\pm 0.00}$ &
  $54.55_{\pm 1.13}$ &
  $81.97_{\pm 3.97}$ &
  $71.02$ \\
 &
  LLaGA-Qwen &
  $68.08_{\pm 3.42}$ &
  $64.04_{\pm 1.91}$ &
  $77.72_{\pm 2.41}$ &
  $55.18_{\pm 1.95}$ &
  $66.26$ &
  $76.60_{\pm 0.24}$ &
  $53.32_{\pm 1.77}$ &
  $83.82_{\pm 1.59}$ &
  $71.25$ \\
 &
  TEA-GLM-Mistral &
  $71.10_{\pm 2.22}$ &
  $65.12_{\pm 12.03}$ &
  \cellcolor[HTML]{FFF2E9}$86.60_{\pm 2.34}$ &
  $61.85_{\pm 3.75}$ &
  $71.17$ &
  $79.73_{\pm 1.96}$ &
  \cellcolor[HTML]{FFF2E9}$60.96_{\pm 2.77}$ &
  $81.23_{\pm 1.91}$ &
  $73.97$ \\
\multirow{-6}{*}{\textbf{\begin{tabular}[c]{@{}c@{}}Graph-\\ LLM\\ Method\end{tabular}}} &
  TEA-GLM-Qwen &
  $66.36_{\pm 6.92}$ &
  $67.88_{\pm 3.22}$ &
  $85.02_{\pm 4.39}$ &
  $60.95_{\pm 2.05}$ &
  $70.05$ &
  $78.40_{\pm 3.98}$ &
  $58.99_{\pm 5.58}$ &
  $84.05_{\pm 1.52}$ &
  $73.81$ \\ \hline
 &
  GLIP-Mistral &
  \cellcolor[HTML]{FFF2E9}$72.45_{\pm 1.45}$ &
  \cellcolor[HTML]{FFDEC7}\textbf{$68.37_{\pm 2.22}$} &
  \cellcolor[HTML]{FFDEC7}\textbf{$87.21_{\pm 1.17}$} &
  \cellcolor[HTML]{FFDEC7}\textbf{$63.72_{\pm 2.68}$} &
  \cellcolor[HTML]{FFDEC7}$72.94$ &
  \cellcolor[HTML]{FFDEC7}\textbf{$80.99_{\pm 1.41}$} &
  \cellcolor[HTML]{FFDEC7}\textbf{$62.63_{\pm 1.80}$} &
  \cellcolor[HTML]{FFDEC7}\textbf{$86.38_{\pm 1.78}$} &
  \cellcolor[HTML]{FFDEC7}$76.67$ \\
\multirow{-2}{*}{\textbf{Ours}} &
  GLIP-Qwen &
  \cellcolor[HTML]{FFDEC7}\textbf{$73.84_{\pm 3.11}$} &
  \cellcolor[HTML]{FFF2E9}$68.37_{\pm 2.41}$ &
  $86.31_{\pm 1.74}$ &
  \cellcolor[HTML]{FFF2E9}$61.95_{\pm 5.40}$ &
  \cellcolor[HTML]{FFF2E9}$72.62$ &
  \cellcolor[HTML]{FFF2E9}$80.04_{\pm 1.20}$ &
  $59.50_{\pm 2.43}$ &
  \cellcolor[HTML]{FFF2E9}$84.17_{\pm 1.04}$ &
  \cellcolor[HTML]{FFF2E9}$74.57$ \\ \hline
\end{tabular}
\end{table*}

\noindent\textbf{Baselines.} 
We compare GLIP with a wide range of representative baselines, including both graph pretraining methods and graph–LLM methods. The graph pretraining baselines cover four classical self-supervised graph learning algorithms {GraphMAE}~\cite{hou2022graphmae}, {GraphCL}~\cite{you2020graph},  {DGI}~\cite{velivckovic2018deep}, and {SimGRACE}~\cite{xia2022simgrace}. These algorithms enhance graph structures and features through contrastive learning or masked reconstruction objectives to learn expressive representations without supervision. We also consider graph–LLM methods include recent frameworks that integrate large language models with graph encoders, namely {GraphGPT}~\cite{tang2024graphgpt}, {LLaGA}~\cite{chen2024llaga}, and {TEA-GLM}~\cite{wang2024llms}.  Given that the three methods are tailored to node-level settings and, to our knowledge, no specialized graph-level graph-LLM method is available, we derive graph-level embeddings by applying average pooling to the node representations for all three methods.

\noindent\textbf{Implementations.} 
We evaluate two metrics: Accuracy and Macro F1-score. Each result is averaged over five runs, and both the mean and standard deviation are reported. We adopt GCN~\cite{kipf2016semi} as the graph encoder backbone. Each graph–LLM model is evaluated under two instruction-tuned backbones, Mistral-7B-Instruct-v0.2~\cite{jiang2023mistral7b} (Mistral) and Qwen2.5-7B-Instruct~\cite{bai2023qwen} (Qwen), to ensure fairness across architectures. In semi-supervised classification, we use a data split of 5\% for training, 15\% for validation, and 80\% for testing. All experiments are conducted on a single NVIDIA H100 GPU with 80 GB of memory. Our source code is publicly available on \url{https://anonymous.4open.science/r/GLIP}. More implementation details are provided in the Appendix.

\subsection{Semi-Supervised Graph Classification (RQ1)}

\noindent\textbf{Results and Analysis.}
Under the semi-supervised setting, we adopt a 5\%/15\%/80\% split for training, validation, and testing, respectively. With limited supervision, pretraining-based methods demonstrate clear advantages. As shown in Table~\ref{tb-semi-accuracy} (Accuracy) and Table~\ref{tb-semi-f1} in the Appendix (F1 score), classical graph pretraining approaches already surpass early Graph–LLM pipelines (e.g., LLaGA, GraphGPT), highlighting the importance of robust structural representation learning under label scarcity. Our proposed {GLIP} pushes this further and achieves the best overall performance in both Accuracy and Macro-F1. For example, on feature graphs, {GLIP-Mistral} attains an average {accuracy} of 72.94\% (Table~\ref{tb-semi-accuracy}); on TAG datasets, it reaches an average {accuracy} of 76.67\%.

On the four feature-only datasets (MUTAG, PROTEINS, REDDIT, IMDB), GraphGPT is inapplicable due to the lack of textual inputs.  GLIP yields consistent gains across all datasets. For instance, \textbf{GLIP-Qwen} achieves 73.84\% on MUTAG and 68.37\% on PROTEINS, outperforming LLaGA and TEA-GLM by roughly 2–10 percentage points in accuracy. On the three text-attributed graphs (BBBP, BACE, E-Com), GLIP maintains its lead over both graph pretraining and graph–LLM counterparts under the same LLM backbone. With Mistral, GLIP attains 80.99\% (BBBP), 62.63\% (BACE), and 86.38\% (E-Com) in accuracy, consistently surpassing TEA-GLM, LLaGA, and GraphGPT. A similar pattern holds with Qwen, where GLIP matches or exceeds the strongest baselines on all datasets in both accuracy and F1. These results indicate that integrating global and local representations, together with diffusion-enhanced patches and LLM-guided supervision, enables GLIP to capture richer structural and semantic contexts across both {feature graphs} and  {text-attributed graphs}, leading to strong and consistent performance even under limited supervision.

\subsection{Few Shot Graph Classification  (RQ2)}

\begin{table*}[htbp!]
\centering
\fontsize{8.5}{12}\selectfont  
\caption{Graph reasoning performance compared with GPT-4o outputs with \colorbox[HTML]{FFDEC7}{\textbf{best}} and \colorbox[HTML]{FFF2E9}{second-best} results highlighted. }
\label{tb-bert}
\begin{tabular}{ccccccccccc}\hline
 &
  \multicolumn{3}{c}{BERT Score} &
  \multicolumn{4}{c}{BLEU} &
  \multicolumn{3}{c}{ROUGE} \\
\multirow{-2}{*}{Model} &
  Precision &
  Recall &
  F1 &
  BLEU-1 &
  BLEU-2 &
  BLEU-3 &
  BLEU-4 &
  ROUGE-1 &
  ROUGE-2 &
  ROUGE-L \\ \hline
GraphGPT &
  78.41 &
  73.52 &
  75.87 &
  45.42 &
  30.51 &
  22.04 &
  16.89 &
  48.52 &
  21.91 &
  46.45 \\
LLaGA &
  \cellcolor[HTML]{FFF2E9}80.23 &
  \cellcolor[HTML]{FFF2E9}75.44 &
  \cellcolor[HTML]{FFF2E9}77.76 &
  47.35 &
  34.82 &
  25.63 &
  19.97 &
  \cellcolor[HTML]{FFF2E9}50.80 &
  24.01 &
  49.62 \\
TEA-GLM &
  79.66 &
  75.12 &
  77.32 &
  \cellcolor[HTML]{FFF2E9}47.98 &
  \cellcolor[HTML]{FFF2E9}35.35 &
  \cellcolor[HTML]{FFF2E9}26.72 &
  \cellcolor[HTML]{FFF2E9}20.64 &
  49.87 &
  \cellcolor[HTML]{FFF2E9}25.42 &
  \cellcolor[HTML]{FFF2E9}50.90 \\
GLIP &
  \cellcolor[HTML]{FFDEC7}82.71 &
  \cellcolor[HTML]{FFDEC7}81.52 &
  \cellcolor[HTML]{FFDEC7}82.11 &
  \cellcolor[HTML]{FFDEC7}54.71 &
  \cellcolor[HTML]{FFDEC7}41.21 &
  \cellcolor[HTML]{FFDEC7}30.79 &
  \cellcolor[HTML]{FFDEC7}24.68 &
  \cellcolor[HTML]{FFDEC7}61.82 &
  \cellcolor[HTML]{FFDEC7}34.21 &
  \cellcolor[HTML]{FFDEC7}59.78 \\ \hline
\end{tabular}
\end{table*}

The few-shot setting presents an even more challenging scenario than the semi-supervised one. Here, we report the results of {5-shot} classification, where only five labeled samples per class are available for training, while the results of the 3-shot setting are provided in the Appendix. This extreme data scarcity severely constrains the effectiveness of conventional graph pretraining and graph–LLM methods. Nevertheless, {GLIP} exhibits remarkable robustness under this condition, achieving consistent and substantial improvements on both {feature graphs} and {text-attributed graphs} in terms of accuracy and F1 score, as shown in Table~\ref{tb-5-accuracy} and Table~\ref{tb-5-F1} (Appendix).

On the {feature graph datasets} (MUTAG, PROTEINS, REDDIT, IMDB), GLIP significantly outperforms the compared baselines. Specifically, {GLIP-Mistral} attains an average accuracy of 66.73\% and F1 of 65.17\%, while {GLIP-Qwen} further reaches 67.33\% accuracy and 65.48\% F1. These results consistently surpass TEA-GLM, LLaGA, and classical graph pretraining methods such as GraphCL and DGI. 

On the {text-attributed datasets} (BBBP, BACE, E-Com), GLIP-Qwen achieves either the best or second-best results across all methods. For example, {GLIP-Qwen} records 74.85\% accuracy on BBBP and 77.04\% on E-Com, outperforming other graph–LLM frameworks like TEA-GLM and GraphGPT. These results demonstrate that, even under the extreme few-shot regime, the proposed patch selection mechanism and diffusion-based contextual encoding substantially enhance representation quality, enabling GLIP to extract more discriminative features from very limited supervision.

\subsection{Graph Reasoning Task (RQ3)}

\begin{figure*}[htbp!]
  \centering
  \includegraphics[width=1\linewidth]{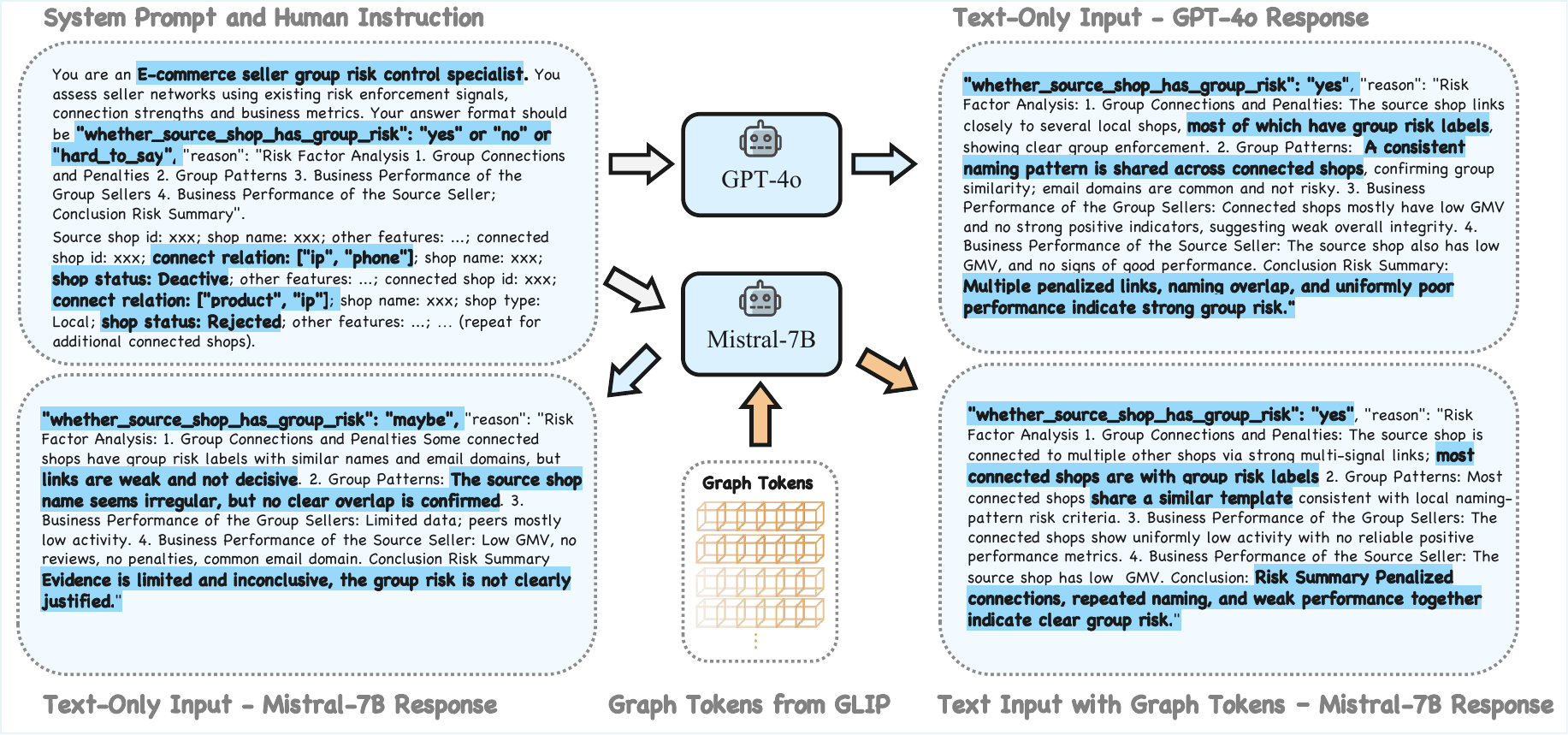}
  \caption{Case Study: Comparative Responses of GPT-4o and Mistral-7B with and without Graph Tokens from GLIP.}
  \label{fig:case}
  \vspace{-3mm}
\end{figure*}




To further assess the reasoning capability of GLIP, we conduct a graph-level reasoning task on the  {E-Com} dataset using  {Mistral-7B}. Following prior Graph-LLM baselines such as GraphGPT, TEA-GLM, and LLaGA, we report results with the LLM backbone (Mistral-7B) kept frozen. Additional experiments with an unfrozen LLM are provided in the appendix. In this task, each input corresponds to a  source seller  connected to multiple neighboring shops through various relational attributes such as shared IPs, product images, or registration devices. The model is required to infer whether the source shop exhibits potential group risk and whether it should be deactivated, strictly following the pre-defined structured output format: {"whether source shop has group risk": "yes/no/hard\_to\_say", "reason": "Risk Factor Analysis 1--4; Conclusion Risk Summary."} Two examples are provided to guide the model's reasoning and response formatting.

To construct reliable supervision for this reasoning task, we first query GPT-4o with the same prompt and input descriptions for each seller case, and treat GPT-4o's structured responses including both the final decision and the detailed reasoning as reference answers. These GPT-4o outputs serve as pseudo ground-truth annotations, reflecting high-quality and consistent reasoning patterns over the seller graphs.

We then  compare three configurations: (i) GPT-4o with text input, (ii)Mistral-7B with text input, and (iii) Mistral-7B with graph tokens generated by GLIP. As shown in Figure~\ref{fig:case}, GPT-4o produces correct and well-structured outputs that strictly follow the predefined reasoning format. In contrast, the frozen Mistral-7B model fails to adhere to the required structure and often outputs ambiguous judgments such as “maybe,” reflecting its limited ability to interpret relational dependencies. It also struggles to recognize key graph-based risk patterns, such as repetitive naming conventions among connected shops, which are indicative of potential collusive groups.

After incorporating GLIP’s graph tokens, the reasoning quality of Mistral-7B improves substantially.  The graph tokens provide explicit structural cues that help (1) reduce hallucination and enforce response consistency, (2) enhance structural reasoning and judgment accuracy, and (3) enable smaller models to approach GPT-level analytical reliability in graph reasoning tasks, even when the LLM backbone remains frozen. Table~\ref{tb-bert} reports the quantitative results of the graph reasoning task, where GPT-4o’s outputs are used as the reference standard. We evaluate the models using BERTScore~\cite{zhang2019bertscore}, BLEU~\cite{papineni2002bleu}, and ROUGE~\cite{lin2004rouge} to measure semantic fidelity, linguistic coherence, and structural consistency. As shown, GLIP-Mistral achieves substantial improvements over all baselines, particularly outperforming the text-only Mistral variant by a large margin across all metrics. GLIP achieves a BERTScore-F1 of 0.821 and a ROUGE-L of 0.598, demonstrating that GLIP effectively enables frozen small-scale LLMs to generate structured, semantically faithful, and logically consistent reasoning outputs comparable to much larger models.



\begin{figure}[t]
  \centering
  \includegraphics[width=1\linewidth]{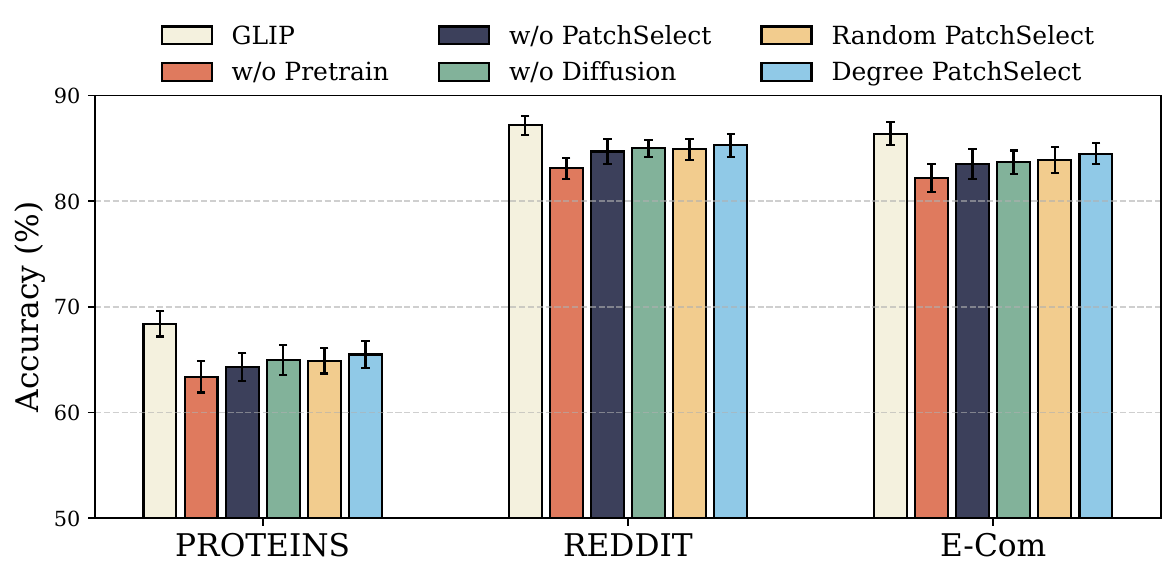}
    \vspace{-3mm}
  \caption{Ablation Results on Key Components of GLIP.}
    \vspace{-3mm}
  \label{fig:ab}
\end{figure}

\subsection{Ablation Study (RQ4)}

Figure~\ref{fig:ab} presents the ablation results of our GLIP framework on three representative datasets (PROTEINS, REDDIT, and E-Com) under the semi-supervised setting. We systematically remove or replace key components to assess their contributions. As shown, removing the pretraining stage ({w/o Pretrain}) leads to a significant accuracy drop across all datasets, indicating that our joint pretraining strategy plays a crucial role in improving downstream performance. Likewise, removing the patch selection module ({w/o PatchSelect}) or the diffusion process ({w/o Diffusion}) consistently degrades performance, demonstrating that both local structural awareness and global information propagation are essential.


We further compare different patch selection strategies. The Random PatchSelect baseline performs markedly worse than our adaptive strategy, confirming the benefit of informed patch construction. Degree PatchSelect, which prioritizes high-degree nodes and their one-hop neighbors, achieves slightly better results than random sampling but still underperforms our method. This indicates that purely structural heuristics are insufficient, while GLIP’s joint structural–feature selection better balances local diversity and global coherence, leading to superior graph understanding.



   \vspace{-2mm}
\section{Conclusion}
In this paper, we proposed GLIP, a Graph-LLM joint pretraining framework for graph-level tasks. GLIP adopts a label-free, self-supervised pretraining scheme that involves large language models as a semantic decision interface, guiding the learning of graph representations to be organized according to their similarities and differences in the LLM token space. The framework introduces a multi-token patch selection strategy to integrate both global and local structural and feature information, and incorporates a diffusion-based projection operator to enhance contextual dependencies among selected patches, enabling more effective semantic understanding of graphs. Building upon this pretrained representation space, GLIP further adapts the model to downstream tasks through task-specific fine-tuning with limited labeled data.

Extensive experiments on both feature-based and text-attributed graphs demonstrate that GLIP consistently outperforms classical graph pretraining methods and recent graph--LLM frameworks under semi-supervised, few-shot, and graph reasoning settings. In future work, we plan to extend GLIP to more complex graph domains, such as heterogeneous graphs, symbolic graphs, and dynamic relational systems.

\clearpage
\newpage

\bibliography{refs,newrefs}
\bibliographystyle{ACM-Reference-Format}
\balance

\newpage
\appendix

\renewcommand{\theequation}{A.\arabic{equation}}
\setcounter{equation}{0}  

\section{Algorithms and Proofs}
In this section, we provide the pseudocode of Algorithm~\ref{alg:greedy} as well as the detailed proofs of our three Theorems.

\subsection{Pseudocode of Algorithm~\ref{alg:greedy}}

 \begin{algorithm}[h]
\caption{$\textsc{GreedyPatchSelect}(\mathcal{G}, \mathcal{C}, k, X, \lambda)$}
\label{alg:greedy}
\Input{
Graph $\mathcal{G}=(\mathcal{V},\mathcal{E})$ with $N$ nodes; 
Candidate patches $\mathcal{C}=\{S_1,\dots,S_t\}$; 
Node features $\XX=[\xx_1,\dots,\xx_N]^\top \in \mathbb{R}^{N\times F}$; 
Budget $k$; 
Weight $\lambda$.
}
\Output{Selected set $U\subseteq\mathcal{C}$ with $|U|=k$.}

\textbf{Initialize:} $U\leftarrow\emptyset$; 
$C_U\leftarrow\emptyset$; 
$M_U(v)\leftarrow 0,\ \forall v\in\mathcal{V}$;

\textbf{Precompute:} For each candidate patch $S_j$ and node $v\in\mathcal{V}$, 
$m_{S_j}(v)\leftarrow\max_{u\in S_j}\langle \xx_v,\xx_u\rangle$;

\For{$r=1$ \KwTo $k$}{
  \For{each $S_j \in \mathcal{C}\setminus U$}{
    $\Delta_{\text{structure}}(S_j)\leftarrow \dfrac{|C_U\cup S_j|-|C_U|}{N}$;

    $M'(v)\leftarrow\max\{M_U(v),m_{S_j}(v)\},\ \forall v\in\mathcal{V}$;

    $\Delta_{\text{feature}}(S_j)\leftarrow\dfrac{1}{N}\sum_{v=1}^N M'(v)-\dfrac{1}{N}\sum_{v=1}^N M_U(v)$;

    $\Delta F(S_j)\leftarrow \lambda\,\Delta_{\text{structure}}(S_j)+(1-\lambda)\,\Delta_{\text{feature}}(S_j)$;
  }

  $S^\star\leftarrow\arg\max_{S_j\in \mathcal{C}\setminus U}\Delta F(S_j)$;

  $U\leftarrow U\cup\{S^\star\}$; 
  $C_U\leftarrow C_U\cup S^\star$; \\
  $M_U(v)\leftarrow \max\{M_U(v),m_{S^\star}(v)\},\ \forall v\in\mathcal{V}$;
}

\Return $U$
\end{algorithm}

\subsection{The proof of Theorem~\ref{th-NP}}



\begin{proof}
We reduce from the classical Maximum Coverage problem, which is NP-hard~\cite{feige1998threshold}.  
Given a Maximum Coverage instance with universe $V=\{v_1, v_2, \dots, v_N\}$, a subset family $\mathcal{S}=\{S_1, \dots, S_t\}$, and a budget $k$, we construct a Patch Selection instance as follows:  
(i) create a graph $\mathcal{G}$ with $N$ nodes, one for each element in {$V$};  
(ii) let the candidate patches be $\mathcal{C}=\{S_1,\dots,S_t\}$;  
(iii) set $\lambda=1$ and keep the same budget $k$.  

In this case, the Patch Selection objective reduces to
\[
\max_{U\subseteq \mathcal{C},\,|U|=k} f_{\mathrm{structure}}(U)
= {\frac{1}{N}\Big|\bigcup_{S\in U} S\Big|},
\]
which is exactly the Maximum Coverage objective up to a constant factor $1/N$.  
Thus, solving Patch Selection solves Maximum Coverage, proving NP-hardness.
\end{proof}

\subsection{The proof of Theorem~\ref{th-mo_and_sub}}
 
\begin{proof}
\emph{Monotonicity.} 
Both $f_{\mathrm{structure}}$ and $f_{\mathrm{feature}}$ are monotone nondecreasing.  
For $f_{\mathrm{structure}}$, adding a new patch can only increase (or leave unchanged) the number of covered nodes.  
For $f_{\mathrm{feature}}$, enlarging the selected set cannot reduce any node’s maximum similarity to the set.  
Since $F(U)$ is a nonnegative linear combination of these two terms, it is also monotone.
 
\emph{Submodularity.}  
Let $A,B\subseteq \mathcal{C}$ with $A\subseteq B$, and let $S_i\in \mathcal{C}\setminus B$.  
We need to show
\[
F(A\cup\{S_i\}) - F(A) \;\ge\; F(B\cup\{S_i\}) - F(B).
\]

\textbf{(1) Structural term.}  
By definition, {let $C_U=\bigcup_{S\in U}S$ and}
\[
{f_{\mathrm{structure}}(U) = \frac{|C_U|}{N}.}
\]
Hence
\[
f_{\mathrm{structure}}(A\cup\{S_i\}) - f_{\mathrm{structure}}(A) 
= {\frac{|S_i\setminus C_A|}{N}},
\]
and
\[
f_{\mathrm{structure}}(B\cup\{S_i\}) - f_{\mathrm{structure}}(B) 
= {\frac{|S_i\setminus C_B|}{N}}.
\]
Since $A\subseteq B$, it follows that {$C_A\subseteq C_B$}, giving
\[
f_{\mathrm{structure}}(A\cup\{S_i\}) - f_{\mathrm{structure}}(A) 
\;\ge\; 
f_{\mathrm{structure}}(B\cup\{S_i\}) - f_{\mathrm{structure}}(B).
\]

\textbf{(2) Feature term.}  
By definition,
\[
{f_{\mathrm{feature}}(U) = \frac{1}{N}\sum_{v=1}^N \max_{u\in C_U} \cos(\xx_v,\xx_u).}
\]
Let 
\[
M_U(v) := {\max_{u\in C_U}} \cos(\xx_v,\xx_u), 
\quad 
m_{S_i}(v) := \max_{u\in S_i}\cos(\xx_v,\xx_u).
\]
Then
\[
f_{\mathrm{feature}}(A\cup\{S_i\}) - f_{\mathrm{feature}}(A) 
= \frac{1}{N}\sum_{v=1}^N \max\{m_{S_i}(v)-M_A(v),0\},
\]
and
\[
f_{\mathrm{feature}}(B\cup\{S_i\}) - f_{\mathrm{feature}}(B) 
= \frac{1}{N}\sum_{v=1}^N \max\{m_{S_i}(v)-M_B(v),0\}.
\]
Since $A\subseteq B$, we have {$C_A\subseteq C_B$ and hence} $M_A(v)\le M_B(v)$ for all $v$, which implies
\[
\max\{m_{S_i}(v)-M_A(v),0\} \;\ge\; \max\{m_{S_i}(v)-M_B(v),0\}.
\]
Summing over all $v$ yields
\[
f_{\mathrm{feature}}(A\cup\{S_i\}) - f_{\mathrm{feature}}(A) 
\;\ge\;
f_{\mathrm{feature}}(B\cup\{S_i\}) - f_{\mathrm{feature}}(B).
\]

\textbf{(3) Conclusion.}  
Both $f_{\mathrm{structure}}$ and $f_{\mathrm{feature}}$ are submodular.  
Since $F(U)$ is a nonnegative linear combination of these terms, $F(U)$ is also submodular.
\end{proof}

\subsection{The proof of Theorem~\ref{th-diffusion}}

\begin{proof}
Stacking \eqref{FJeq} over all nodes gives the matrix recurrence
\begin{equation}
\ZZ(q+1) \;=\; (\II+\alpha \DD)^{-1}\big(\XX+\alpha \AA\,\ZZ(q)\big)
= \JJ\,\ZZ(q)+\BB,
\end{equation}
where $\JJ=(\II+\alpha\DD)^{-1}(\alpha\AA)$ and $\BB=(\II+\alpha\DD)^{-1}\XX$.

For each row $i$, the $\ell_\infty$-row sum of $\JJ$ is
\[
\sum_j |\JJ_{ij}| \;=\; \frac{\alpha d_i}{1+\alpha d_i} \;<\; 1,
\]
hence $\|\JJ\|_\infty<1$ and therefore $\rho(\JJ)<1$. This implies the affine iteration converges linearly to a unique fixed point for any initialization.

At convergence, $\ZZ^\star=\JJ\ZZ^\star+\BB$, i.e.,
\[
(\II-\JJ)\ZZ^\star=\BB.
\]
Multiplying both sides by $(\II+\alpha\DD)$ yields
\[
\big[(\II+\alpha\DD)-\alpha\AA\big]\ZZ^\star=\XX,
\]
which simplifies to $(\II+\alpha\LL)\ZZ^\star=\XX$. Thus
\[
\ZZ^\star=(\II+\alpha\LL)^{-1}\XX,
\]
which finishes the proof.
\end{proof}

\begin{table*}[t!]
\centering
\caption{Training time comparison (seconds per epoch) in the format of $T_{\text{pre}} + T_{\text{fine}}$, where $T_{\text{pre}}$ denotes pretraining time and $T_{\text{fine}}$ denotes fine-tuning time. Methods without pretraining are reported as $0 + T_{\text{fine}}$.}
\label{tab:time_cost}
\begin{tabular}{lccccccc}
\hline
Method & MUTAG & PROTEINS & REDDIT-BINARY & IMDB-BINARY & BBBP & BACE & E-Commerce \\
\hline
GraphGPT & - & - & - & - & 32.1+7.3 & 20.2+6.0 & 15.2+3.5 \\
LLaGA   & 0+1.1 & 0+4.6 & 0+6.3 & 0+3.5 & 0+7.1 & 0+6.3 & 0+2.3 \\
TEA-GLM & 1.2+1.2 & 2.3+4.4 & 2.1+6.3 & 1.9+3.4 & 1.2+6.9 & 1.5+6.2 & 1.2+2.1 \\
GLIP    & 2.6+1.1 & 25.2+4.2 & 38.1+6.1 & 20.2+3.2 & 46.1+7.2 & 30.4+6.2 & 18.6+2.4 \\
\hline
\end{tabular}
\end{table*}

\section{Implementation Details}

We now present the implementation details and parameter settings used in our experiments. For GLIP, the parameter $\lambda$ in Theorem~\ref{th-mo_and_sub} is set to 0.5. We use Algorithm~\ref{alg:greedy} and construct the candidate patch set $\mathcal{C}$ from the one-hop neighbors of the top 20 nodes ranked by degree, and we choose $k = 5$. The diffusion operator is computed using the conjugate gradient method to efficiently solve the associated linear systems. 

During the pretraining stage, we adopt the default augmentation settings: $\texttt{aug\_ratio} = 0.2$, $\texttt{aug1} = \text{"permE"}$, and $\texttt{aug2} = \text{"maskN"}$. The model is trained with a batch size of 16, learning rate of 0,001, for 100 epochs, with the random seed set to 42. The graph encoder is a two-layer GCN with hidden dimension 64 and output dimension 64. In the semi-supervised fine-tuning stage, we set the batch size to 32 and the learning rate to 0.001. Each experiment is repeated with five different random seeds $\{0, 1, 2, 3, 4\}$, and the average performance across these runs is reported. For graph classification tasks using LLMs such as Mistral and Qwen, we employ the model's `generate` method for text generation, with `max\_new\_tokens` set to 3 to constrain output length. Since `do\_sample` is disabled (set to False), deterministic output is ensured, eliminating the need for a temperature parameter to control uncertainty.

For other baselines, we use the default hyperparameter settings when available; otherwise, we perform a grid search to determine the optimal values.

\section{Limitations}
Our study focuses primarily on graph-level tasks, and we have not experimented with node-level or edge-level settings. The proposed patch selection and diffusion modules are specifically tailored for capturing holistic graph semantics rather than fine-grained local predictions. However, the joint pretraining strategy itself is general and can be extended to node or edge representation learning with suitable adaptations.

\section{Additional Experiment Results}
In this section, we present additional experimental results.

\subsection{Training Time Analysis}
Table~\ref{tab:time_cost} reports the per-epoch training time of different Graph-LLM methods, decomposed into the pretraining stage and the fine-tuning stage. As expected, methods without a pretraining phase (e.g., LLaGA) incur zero pretraining cost and only spend time on task-specific fine-tuning, while pretraining-based methods introduce additional overhead in the first stage. TEA-GLM requires less pretraining time since its pretraining stage does not involve the LLM backbone. In contrast, GLIP requires a longer pretraining time compared to GraphGPT, due to its joint optimization of the graph encoder and projection modules, as well as the additional patch selection and diffusion operations.

However, the fine-tuning time of GLIP is comparable to or only slightly higher than that of other baselines across all datasets, indicating that the extra computational cost is mainly concentrated in the one-time pretraining stage. In practical usage, the pretrained GLIP model can be directly reused for different downstream tasks (e.g., classification and reasoning on the E-Commerce dataset) without repeating the pretraining process, therefore the overall pretraining cost is acceptable in practice.

\subsection{Frozen vs. Fine-tuned LLMs for Graph Reasoning}

Table~\ref{tb-bert} compares different training strategies of GLIP on the graph reasoning task, including training without pretraining, keeping the LLM backbone frozen, applying parameter-efficient fine-tuning with LoRA, and full-parameter fine-tuning. These variants allow us to analyze the trade-offs between performance and computational cost.

As shown in the table, removing the pretraining stage leads to a significant performance drop, confirming the necessity of joint pretraining. Keeping the LLM frozen already achieves strong results, demonstrating that GLIP can effectively leverage a fixed LLM through graph token alignment. Applying LoRA further improves performance with only a modest increase in training cost. Full-parameter fine-tuning achieves the best overall performance, but at a substantially higher computational expense.

In our experiments, full fine-tuning requires four 80G H100 GPUs and incurs significantly larger time and memory consumption, making it much less practical in real-world scenarios. In contrast, the frozen and LoRA-based variants provide a much better balance between performance and efficiency, highlighting the practical advantage of GLIP under realistic computational budgets.

\begin{table*}[t!]
\centering
\fontsize{8.5}{12}\selectfont  
\caption{Graph reasoning performance of different GLIP training variants compared with GPT-4o outputs.}
\label{tb-bert}
\begin{tabular}{ccccccccccc}\hline
 &
  \multicolumn{3}{c}{BERT Score} &
  \multicolumn{4}{c}{BLEU} &
  \multicolumn{3}{c}{ROUGE} \\
\multirow{-2}{*}{Model} &
  Precision &
  Recall &
  F1 &
  BLEU-1 &
  BLEU-2 &
  BLEU-3 &
  BLEU-4 &
  ROUGE-1 &
  ROUGE-2 &
  ROUGE-L \\ \hline

GLIP (w/o Pretrain) &
 80.12 &
  75.03 &
  77.54 &
  47.31 &
  33.64 &
  26.11 &
  19.32 &
  50.27 &
  24.84 &
  49.91 \\

GLIP (LLM Frozen) &
  82.71 &
  81.52 &
  82.11 &
  54.71 &
  41.21 &
  30.79 &
  24.68 &
  61.82 &
  34.21 &
  59.78 \\

GLIP (LoRA) &
  83.45 &
  82.36 &
  82.90 &
  55.38 &
  42.07 &
  31.54 &
  25.41 &
  62.73 &
  35.02 &
  60.81 \\

GLIP (Full Finetune) &
  85.12 &
  84.03 &
  84.57 &
  57.46 &
  44.18 &
  33.12 &
  27.05 &
  64.89 &
  37.26 &
  62.34 \\ \hline

\end{tabular}
\end{table*}

\section{Additional Classification Results}
This section reports additional classification results, including the F1 scores for semi-supervised learning and 5-shot settings, as well as the accuracy and F1 scores under the 3-shot scenario.

\begin{table*}[t]
\centering
\caption{\textbf{Performance comparison under semi-supervised settings  with F1 Score (\%) reported.} 
The \colorbox[HTML]{FFDEC7}{\textbf{best}} and \colorbox[HTML]{FFF2E9}{second-best} results are highlighted. Results are averaged over five runs. }
\label{tb-semi-f1}
\fontsize{8.5}{13}\selectfont  
\setlength{\tabcolsep}{4pt} 
\begin{tabular}{cc|ccccc|cccc}
\hline
\multicolumn{2}{c|}{} &
  \multicolumn{5}{c|}{\textbf{Feature Graphs}} &
  \multicolumn{4}{c}{\textbf{Text-Attributed Graphs}} \\
\multicolumn{2}{c|}{\multirow{-2}{*}{\textbf{\begin{tabular}[c]{@{}c@{}}F1 Score\\ Semi-Supervised\end{tabular}}}} &
  \textbf{MUTAG} &
  \textbf{PROTEINS} &
  \textbf{REDDIT} &
  \textbf{IMDB} &
  \textbf{Average} &
  \textbf{BBBP} &
  \textbf{BACE} &
  \textbf{E-Commerce} &
  \textbf{Average} \\ \hline
 &
  GraphMAE &
  $44.50_{\pm 9.31}$ &
  $44.77_{\pm 7.54}$ &
  $80.10_{\pm 4.68}$ &
  $42.84_{\pm 9.58}$ &
  $53.05$ &
  $43.35_{\pm 0.01}$ &
  $36.07_{\pm 1.71}$ &
  $83.37_{\pm 1.23}$ &
  $54.26$ \\
 &
  GraphCL &
  $61.45_{\pm 9.83}$ &
  $61.69_{\pm 4.08}$ &
  $86.19_{\pm 2.02}$ &
  $60.47_{\pm 2.65}$ &
  $67.45$ &
  $48.01_{\pm 8.00}$ &
  $53.12_{\pm 8.37}$ &
  $83.12_{\pm 2.03}$ &
  $61.42$ \\
 &
  DGI &
  $51.70_{\pm 9.53}$ &
  $46.87_{\pm 9.77}$ &
  $84.91_{\pm 2.20}$ &
  $51.03_{\pm 7.17}$ &
  $58.63$ &
  $43.35_{\pm 0.00}$ &
  $47.65_{\pm 9.70}$ &
  $82.09_{\pm 2.04}$ &
  $57.70$ \\
\multirow{-4}{*}{\textbf{\begin{tabular}[c]{@{}c@{}}Graph\\ Pretrain\\ Method\end{tabular}}} &
  SimGRACE &
  $47.77_{\pm 9.74}$ &
  $46.77_{\pm 7.36}$ &
  $83.31_{\pm 2.37}$ &
  $60.81_{\pm 3.34}$ &
  $59.67$ &
  $43.31_{\pm 0.08}$ &
  $42.75_{\pm 9.56}$ &
  $83.16_{\pm 1.28}$ &
  $56.41$ \\ \hline
 &
  GraphGPT-Mistral &
  - &
  - &
  - &
  - &
  - &
  $52.46_{\pm 9.77}$ &
  $45.13_{\pm 9.24}$ &
  $80.21_{\pm 2.33}$ &
  $59.27$ \\
 &
  GraphGPT-Qwen &
  - &
  - &
  - &
  - &
  - &
  $51.32_{\pm 9.23}$ &
  $46.71_{\pm 9.24}$ &
  $81.62_{\pm 2.13}$ &
  $59.88$ \\
 &
  LLaGA-Mistral &
  $50.87_{\pm 9.17}$ &
  $39.57_{\pm 4.99}$ &
  $75.74_{\pm 3.58}$ &
  $43.20_{\pm 9.49}$ &
  $52.35$ &
  $43.35_{\pm 0.00}$ &
  $42.85_{\pm 9.56}$ &
  $81.74_{\pm 4.43}$ &
  $55.98$ \\
 &
  LLaGA-Qwen &
  $51.64_{\pm 9.43}$ &
  $54.65_{\pm 6.95}$ &
  $77.69_{\pm 2.44}$ &
  $54.16_{\pm 2.24}$ &
  $59.54$ &
  $44.85_{\pm 3.24}$ &
  $38.98_{\pm 6.54}$ &
  $83.80_{\pm 1.60}$ &
  $55.88$ \\
 &
  TEA-GLM-Mistral &
  $63.68_{\pm 2.80}$ &
  $60.42_{\pm 9.77}$ &
  \cellcolor[HTML]{FFF2E9}$86.54_{\pm 2.40}$ &
  $61.34_{\pm 3.74}$ &
  $68.00$ &
  $65.12_{\pm 1.30}$ &
  \cellcolor[HTML]{FFF2E9}$59.66_{\pm 4.57}$ &
  $84.01_{\pm 2.12}$ &
  $69.60$ \\
\multirow{-6}{*}{\textbf{\begin{tabular}[c]{@{}c@{}}Graph-\\ LLM\\ Method\end{tabular}}} &
  TEA-GLM-Qwen &
  $51.82_{\pm 9.96}$ &
  $63.12_{\pm 3.83}$ &
  $84.87_{\pm 4.71}$ &
  $60.76_{\pm 2.18}$ &
  $65.14$ &
  $60.15_{\pm 9.68}$ &
  $58.07_{\pm 5.63}$ &
  $84.02_{\pm 1.54}$ &
  $67.41$ \\ \hline
 &
  GLIP-Mistral &
  \cellcolor[HTML]{FFF2E9}$64.35_{\pm 3.57}$ &
  \cellcolor[HTML]{FFDEC7}$65.99_{\pm 2.42}$ &
  \cellcolor[HTML]{FFDEC7}$86.70_{\pm 1.24}$ &
  \cellcolor[HTML]{FFDEC7}$62.23_{\pm 1.02}$ &
  \cellcolor[HTML]{FFDEC7}$69.82$ &
  \cellcolor[HTML]{FFDEC7}$68.04_{\pm 1.30}$ &
  \cellcolor[HTML]{FFDEC7}$62.42_{\pm 1.80}$ &
  \cellcolor[HTML]{FFDEC7}$86.36_{\pm 1.80}$ &
  \cellcolor[HTML]{FFDEC7}$72.27$ \\
\multirow{-2}{*}{\textbf{Ours}} &
  GLIP-Qwen &
  \cellcolor[HTML]{FFDEC7}$67.43_{\pm 6.78}$ &
  \cellcolor[HTML]{FFF2E9}$64.09_{\pm 4.38}$ &
  $85.58_{\pm 1.77}$ &
  \cellcolor[HTML]{FFF2E9}$61.60_{\pm 5.45}$ &
  \cellcolor[HTML]{FFF2E9}$69.68$ &
  \cellcolor[HTML]{FFF2E9}$67.55_{\pm 1.53}$ &
  $58.84_{\pm 2.88}$ &
  \cellcolor[HTML]{FFF2E9}$84.16_{\pm 1.04}$ &
  \cellcolor[HTML]{FFF2E9}$70.18$ \\ \hline
\end{tabular}
\end{table*}

\begin{table*}[htbp!]
\centering
\fontsize{8.5}{12.5}\selectfont  
\setlength{\tabcolsep}{4pt} 
\caption{\textbf{Performance comparison under 5-shot settings  with Accuracy (\%) reported.} 
The \colorbox[HTML]{FFDEC7}{\textbf{best}} and \colorbox[HTML]{FFF2E9}{second-best} results are highlighted. Results are averaged over five runs. }
\label{tb-5-accuracy}
\begin{tabular}{cc|ccccc|cccc}
\hline
\multicolumn{2}{c|}{} &
  \multicolumn{5}{c|}{\textbf{Feature Graphs}} &
  \multicolumn{4}{c}{\textbf{Text-Attributed Graphs}} \\
\multicolumn{2}{c|}{\multirow{-2}{*}{\textbf{\begin{tabular}[c]{@{}c@{}}Accuracy\\ 5-Shot\end{tabular}}}} &
  \textbf{MUTAG} &
  \textbf{PROTEINS} &
  \textbf{REDDIT} &
  \textbf{IMDB} &
  \textbf{Average} &
  \textbf{BBBP} &
  \textbf{BACE} &
  \textbf{E-Com} &
  \textbf{Average} \\ \hline
 &
  GraphMAE &
  $69.34_{\pm 3.26}$ &
  $59.79_{\pm 0.17}$ &
  $64.17_{\pm 8.92}$ &
  $50.45_{\pm 3.15}$ &
  $60.94$ &
  $72.50_{\pm 0.28}$ &
  $54.36_{\pm 0.03}$ &
  $72.35_{\pm 2.85}$ &
  $66.40$ \\
 &
  GraphCL &
  $73.29_{\pm 6.07}$ &
  $62.05_{\pm 4.45}$ &
  $73.95_{\pm 5.96}$ &
  $53.07_{\pm 6.53}$ &
  $65.59$ &
  $63.83_{\pm 8.26}$ &
  $53.60_{\pm 3.25}$ &
  $73.06_{\pm 9.05}$ &
  $63.50$ \\
 &
  DGI &
  $70.53_{\pm 3.93}$ &
  $61.07_{\pm 2.57}$ &
  $68.55_{\pm 5.60}$ &
  $52.38_{\pm 2.70}$ &
  $63.13$ &
  \cellcolor[HTML]{FFF2E9}$72.54_{\pm 0.62}$ &
  $54.76_{\pm 0.91}$ &
  $73.71_{\pm 4.79}$ &
  $67.00$ \\
\multirow{-4}{*}{\textbf{\begin{tabular}[c]{@{}c@{}}Graph\\ Pretrain\\ Method\end{tabular}}} &
  SimGRACE &
  $72.37_{\pm 2.39}$ &
  $62.39_{\pm 2.71}$ &
  $72.91_{\pm 6.76}$ &
  $50.64_{\pm 4.39}$ &
  $64.58$ &
  $72.14_{\pm 2.04}$ &
  $54.38_{\pm 0.00}$ &
  $72.47_{\pm 5.47}$ &
  $66.33$ \\ \hline
 &
  GraphGPT-Mistral &
  - &
  - &
  - &
  - &
  - &
  $70.16_{\pm 6.37}$ &
  $55.23_{\pm 0.98}$ &
  $70.26_{\pm 3.27}$ &
  $65.22$ \\
 &
  GraphGPT-Qwen &
  - &
  - &
  - &
  - &
  - &
  $69.32_{\pm 1.20}$ &
  $54.18_{\pm 1.08}$ &
  $71.09_{\pm 2.41}$ &
  $64.86$ \\
 &
  LLaGA-Mistral &
  $74.83_{\pm 5.36}$ &
  $60.22_{\pm 0.69}$ &
  $65.39_{\pm 5.32}$ &
  $50.55_{\pm 2.40}$ &
  $62.75$ &
  $72.49_{\pm 0.35}$ &
  $54.95_{\pm 0.80}$ &
  $72.21_{\pm 6.54}$ &
  $66.55$ \\
 &
  LLaGA-Qwen &
  $69.74_{\pm 3.86}$ &
  $58.46_{\pm 9.57}$ &
  $68.14_{\pm 3.95}$ &
  $50.19_{\pm 1.83}$ &
  $61.63$ &
  $71.57_{\pm 0.16}$ &
  $54.07_{\pm 1.24}$ &
  $70.55_{\pm 8.12}$ &
  $65.40$ \\
 &
  TEA-GLM-Mistral &
  $67.63_{\pm 15.43}$ &
  \cellcolor[HTML]{FFF2E9}$63.05_{\pm 5.99}$ &
  $74.24_{\pm 6.79}$ &
  $52.10_{\pm 5.31}$ &
  $64.26$ &
  $66.63_{\pm 7.39}$ &
  $55.18_{\pm 1.92}$ &
  $71.21_{\pm 2.13}$ &
  $64.34$ \\
\multirow{-6}{*}{\textbf{\begin{tabular}[c]{@{}c@{}}Graph-\\ LLM\\ Method\end{tabular}}} &
  TEA-GLM-Qwen &
  $68.81_{\pm 7.17}$ &
  $56.99_{\pm 11.62}$ &
  $74.39_{\pm 6.84}$ &
  \cellcolor[HTML]{FFF2E9}$53.28_{\pm 6.27}$ &
  $63.37$ &
  $70.49_{\pm 5.52}$ &
  $51.65_{\pm 4.97}$ &
  $73.88_{\pm 5.94}$ &
  $65.34$ \\ \hline
 &
  GLIP-Mistral &
  \cellcolor[HTML]{FFDEC7}$76.05_{\pm 2.44}$ &
  $62.66_{\pm 8.85}$ &
  \cellcolor[HTML]{FFDEC7}$75.12_{\pm 3.20}$ &
  $53.09_{\pm 7.53}$ &
  \cellcolor[HTML]{FFF2E9}$66.73$ &
  $72.16_{\pm 1.66}$ &
  \cellcolor[HTML]{FFDEC7}$56.17_{\pm 2.21}$ &
  \cellcolor[HTML]{FFF2E9}$76.60_{\pm 1.82}$ &
  \cellcolor[HTML]{FFF2E9}$68.31$ \\
\multirow{-2}{*}{\textbf{Ours}} &
  GLIP-Qwen &
  \cellcolor[HTML]{FFF2E9}$75.66_{\pm 3.98}$ &
  \cellcolor[HTML]{FFDEC7}$64.07_{\pm 5.57}$ &
  \cellcolor[HTML]{FFF2E9}$75.08_{\pm 3.52}$ &
  \cellcolor[HTML]{FFDEC7}$54.51_{\pm 6.08}$ &
  \cellcolor[HTML]{FFDEC7}$67.33$ &
  \cellcolor[HTML]{FFDEC7}$74.85_{\pm 1.14}$ &
  \cellcolor[HTML]{FFF2E9}$55.14_{\pm 1.66}$ &
  \cellcolor[HTML]{FFDEC7}$77.04_{\pm 5.93}$ &
  \cellcolor[HTML]{FFDEC7}$69.01$ \\ \hline
\end{tabular}
\end{table*}

\begin{table*}[t]
\centering
\fontsize{8.5}{13}\selectfont  
\setlength{\tabcolsep}{4pt} 
\caption{\textbf{Performance comparison under 5-shot settings  with F1 score (\%) reported.} 
The \colorbox[HTML]{FFDEC7}{\textbf{best}} and \colorbox[HTML]{FFF2E9}{second-best} results are highlighted. Results are averaged over five runs. }
\label{tb-5-F1}
\begin{tabular}{cc|ccccc|cccc}
\hline
\multicolumn{2}{c|}{} &
  \multicolumn{5}{c|}{\textbf{Feature Graphs}} &
  \multicolumn{4}{c}{\textbf{Text-Attributed Graphs}} \\
\multicolumn{2}{c|}{\multirow{-2}{*}{\textbf{\begin{tabular}[c]{@{}c@{}}F1 Score\\ 5-Shot\end{tabular}}}} &
  \textbf{MUTAG} &
  \textbf{PROTEINS} &
  \textbf{REDDIT} &
  \textbf{IMDB} &
  \textbf{Average} &
  \textbf{BBBP} &
  \textbf{BACE} &
  \textbf{E-Com} &
  \textbf{Average} \\ \hline
 &
  GraphMAE &
  $59.75_{\pm 9.10}$ &
  $37.70_{\pm 0.65}$ &
  $59.40_{\pm 9.67}$ &
  $42.84_{\pm 4.22}$ &
  $49.92$ &
  $43.92_{\pm 1.06}$ &
  $35.22_{\pm 0.02}$ &
  $75.28_{\pm 2.95}$ &
  $51.47$ \\
 &
  GraphCL &
  $70.00_{\pm 5.60}$ &
  $58.58_{\pm 5.61}$ &
  $73.72_{\pm 6.00}$ &
  $50.96_{\pm 7.46}$ &
  $63.32$ &
  $54.86_{\pm 5.84}$ &
  $49.23_{\pm 7.48}$ &
  $72.73_{\pm 9.35}$ &
  $58.94$ \\
 &
  DGI &
  $65.43_{\pm 7.61}$ &
  $49.28_{\pm 9.81}$ &
  $67.53_{\pm 6.68}$ &
  $43.57_{\pm 7.65}$ &
  $56.45$ &
  $49.08_{\pm 6.31}$ &
  $43.12_{\pm 9.70}$ &
  $74.23_{\pm 5.45}$ &
  $55.48$ \\
\multirow{-4}{*}{\textbf{\begin{tabular}[c]{@{}c@{}}Graph\\ Pretrain\\ Method\end{tabular}}} &
  SimGRACE &
  $69.29_{\pm 4.67}$ &
  $49.31_{\pm 8.85}$ &
  $71.66_{\pm 8.47}$ &
  $45.04_{\pm 3.28}$ &
  $58.83$ &
  $52.39_{\pm 7.40}$ &
  $35.23_{\pm 0.00}$ &
  $73.30_{\pm 5.55}$ &
  $53.64$ \\ \hline
 &
  GraphGPT-Mistral &
  - &
  - &
  - &
  - &
  - &
  $56.67_{\pm 7.06}$ &
  $42.86_{\pm 9.43}$ &
  $72.92_{\pm 3.62}$ &
  $57.48$ \\
 &
  GraphGPT-Qwen &
  - &
  - &
  - &
  - &
  - &
  $54.67_{\pm 7.06}$ &
  $45.86_{\pm 9.43}$ &
  $72.73_{\pm 1.52}$ &
  $57.75$ \\
 &
  LLaGA-Mistral &
  $65.85_{\pm 9.20}$ &
  $43.52_{\pm 6.32}$ &
  $63.45_{\pm 7.68}$ &
  $41.82_{\pm 9.79}$ &
  $53.66$ &
  $43.34_{\pm 0.11}$ &
  $39.95_{\pm 6.47}$ &
  $72.64_{\pm 7.45}$ &
  $51.98$ \\
 &
  LLaGA-Qwen &
  $57.68_{\pm 9.53}$ &
  $47.52_{\pm 9.93}$ &
  $67.65_{\pm 4.50}$ &
  $38.67_{\pm 6.23}$ &
  $52.88$ &
  $43.37_{\pm 0.05}$ &
  $40.46_{\pm 7.94}$ &
  $68.62_{\pm 8.08}$ &
  $50.82$ \\
 &
  TEA-GLM-Mistral &
  $65.21_{\pm 9.29}$ &
  $56.03_{\pm 8.18}$ &
  $72.03_{\pm 6.92}$ &
  $46.70_{\pm 5.10}$ &
  $59.99$ &
  $57.32_{\pm 3.68}$ &
  $48.78_{\pm 8.28}$ &
  $69.62_{\pm 9.90}$ &
  $58.91$ \\
\multirow{-6}{*}{\textbf{\begin{tabular}[c]{@{}c@{}}Graph-\\ LLM\\ Method\end{tabular}}} &
  TEA-GLM-Qwen &
  $61.07_{\pm 9.97}$ &
  $54.89_{\pm 8.75}$ &
  $72.29_{\pm 6.78}$ &
  \cellcolor[HTML]{FFF2E9}$52.09_{\pm 7.34}$ &
  $60.09$ &
  $55.21_{\pm 7.14}$ &
  \cellcolor[HTML]{FFF2E9}$49.55_{\pm 6.01}$ &
  $73.15_{\pm 6.88}$ &
  $59.30$ \\ \hline
 &
  GLIP-Mistral &
  \cellcolor[HTML]{FFF2E9}$73.37_{\pm 2.32}$ &
  \cellcolor[HTML]{FFDEC7}$60.57_{\pm 8.21}$ &
  \cellcolor[HTML]{FFDEC7}$75.18_{\pm 3.57}$ &
  $51.56_{\pm 8.09}$ &
  \cellcolor[HTML]{FFF2E9}$65.17$ &
  \cellcolor[HTML]{FFDEC7}$59.81_{\pm 4.05}$ &
  \cellcolor[HTML]{FFDEC7}$51.00_{\pm 7.76}$ &
  \cellcolor[HTML]{FFDEC7}$77.82_{\pm 6.32}$ &
  \cellcolor[HTML]{FFDEC7}$62.88$ \\
\multirow{-2}{*}{\textbf{Ours}} &
  GLIP-Qwen &
  \cellcolor[HTML]{FFDEC7}$72.67_{\pm 5.34}$ &
  \cellcolor[HTML]{FFF2E9}$60.60_{\pm 6.28}$ &
  \cellcolor[HTML]{FFF2E9}$74.75_{\pm 3.66}$ &
  \cellcolor[HTML]{FFDEC7}$53.90_{\pm 5.25}$ &
  \cellcolor[HTML]{FFDEC7}$65.48$ &
  \cellcolor[HTML]{FFF2E9}$58.32_{\pm 5.53}$ &
  $47.58_{\pm 5.27}$ &
  \cellcolor[HTML]{FFF2E9}$77.67_{\pm 5.76}$ &
  \cellcolor[HTML]{FFF2E9}$61.19$ \\ \hline
\end{tabular}
\end{table*}

\begin{table*}[htbp!]
\centering
\fontsize{8.5}{13}\selectfont  
\setlength{\tabcolsep}{4pt} 
\caption{\textbf{Performance comparison under 3-shot settings  with Accuracy (\%) reported.} 
The \colorbox[HTML]{FFDEC7}{\textbf{best}} and \colorbox[HTML]{FFF2E9}{second-best} results are highlighted. Results are averaged over five runs. }
\label{tb-3-acc}
\begin{tabular}{cc|ccccc|cccc}
\hline
\multicolumn{2}{c|}{} &
  \multicolumn{5}{c|}{\textbf{Feature Graphs}} &
  \multicolumn{4}{c}{\textbf{Text-Attributed Graphs}} \\
\multicolumn{2}{c|}{\multirow{-2}{*}{\textbf{\begin{tabular}[c]{@{}c@{}}Accuracy\\ 3-Shot\end{tabular}}}} &
  \textbf{MUTAG} &
  \textbf{PROTEINS} &
  \textbf{REDDIT} &
  \textbf{IMDB} &
  \textbf{Average} &
  \textbf{BBBP} &
  \textbf{BACE} &
  \textbf{E-Commerce} &
  \textbf{Average} \\ \hline
 &
  GraphMAE &
  $69.42_{\pm 2.10}$ &
  $59.98_{\pm 0.72}$ &
  $62.41_{\pm 7.47}$ &
  $51.72_{\pm 3.02}$ &
  $60.88$ &
  $70.92_{\pm 1.08}$ &
  $54.31_{\pm 0.03}$ &
  \cellcolor[HTML]{FFF2E9}$73.86_{\pm 3.15}$ &
  $66.36$ \\
 &
  GraphCL &
  $71.25_{\pm 2.66}$ &
  $61.21_{\pm 2.01}$ &
  $70.53_{\pm 3.22}$ &
  $51.43_{\pm 3.96}$ &
  $63.61$ &
  $64.92_{\pm 8.83}$ &
  $54.02_{\pm 2.69}$ &
  $69.70_{\pm 7.99}$ &
  $62.88$ \\
 &
  DGI &
  $72.90_{\pm 5.94}$ &
  $62.55_{\pm 3.16}$ &
  $68.48_{\pm 6.18}$ &
  $50.91_{\pm 2.14}$ &
  $63.71$ &
  $70.13_{\pm 0.39}$ &
  $54.08_{\pm 0.85}$ &
  $72.53_{\pm 5.55}$ &
  $65.58$ \\
\multirow{-4}{*}{\textbf{\begin{tabular}[c]{@{}c@{}}Graph\\ Pretrain\\ Method\end{tabular}}} &
  SimGRACE &
  \cellcolor[HTML]{FFF2E9}$74.32_{\pm 4.99}$ &
  $62.08_{\pm 1.94}$ &
  $71.23_{\pm 5.42}$ &
  $51.43_{\pm 1.24}$ &
  $64.77$ &
  $70.29_{\pm 2.73}$ &
  $54.33_{\pm 1.98}$ &
  $71.16_{\pm 8.98}$ &
  $65.26$ \\ \hline
 &
  GraphGPT-Mistral &
  - &
  - &
  - &
  - &
  - &
  $69.23_{\pm 1.49}$ &
  $51.35_{\pm 2.14}$ &
  $69.70_{\pm 7.99}$ &
  $63.43$ \\
 &
  GraphGPT-Qwen &
  - &
  - &
  - &
  - &
  - &
  $68.98_{\pm 2.52}$ &
  $52.35_{\pm 1.23}$ &
  $70.25_{\pm 3.16}$ &
  $63.86$ \\
 &
  LLaGA-Mistral &
  $61.10_{\pm 9.95}$ &
  $59.75_{\pm 0.42}$ &
  $65.30_{\pm 4.98}$ &
  $51.15_{\pm 3.29}$ &
  $59.33$ &
  $70.61_{\pm 3.03}$ &
  $54.38_{\pm 2.94}$ &
  $70.26_{\pm 7.38}$ &
  $65.08$ \\
 &
  LLaGA-Qwen &
  $68.52_{\pm 4.19}$ &
  $62.06_{\pm 2.59}$ &
  $66.36_{\pm 6.33}$ &
  $50.77_{\pm 2.46}$ &
  $61.93$ &
  $70.39_{\pm 2.67}$ &
  $53.16_{\pm 2.62}$ &
  $68.95_{\pm 7.75}$ &
  $64.17$ \\
 &
  TEA-GLM-Mistral &
  $71.68_{\pm 2.11}$ &
  $62.36_{\pm 3.98}$ &
  $70.31_{\pm 4.83}$ &
  $52.71_{\pm 1.92}$ &
  $64.27$ &
  $64.92_{\pm 6.11}$ &
  $53.79_{\pm 6.27}$ &
  $69.71_{\pm 5.29}$ &
  $62.81$ \\
\multirow{-6}{*}{\textbf{\begin{tabular}[c]{@{}c@{}}Graph-\\ LLM\\ Method\end{tabular}}} &
  TEA-GLM-Qwen &
  $68.26_{\pm 2.68}$ &
  $61.32_{\pm 5.33}$ &
  $68.67_{\pm 7.38}$ &
  $51.56_{\pm 5.66}$ &
  $62.45$ &
  $68.83_{\pm 8.24}$ &
  $50.79_{\pm 5.69}$ &
  $68.43_{\pm 7.20}$ &
  $62.68$ \\ \hline
 &
  GLIP-Mistral &
  $73.29_{\pm 5.33}$ &
  \cellcolor[HTML]{FFF2E9}$62.98_{\pm 2.04}$ &
  \cellcolor[HTML]{FFDEC7}$72.45_{\pm 4.91}$ &
  \cellcolor[HTML]{FFF2E9}$53.17_{\pm 3.20}$ &
  \cellcolor[HTML]{FFF2E9}$65.47$ &
  \cellcolor[HTML]{FFDEC7}$72.33_{\pm 2.54}$ &
  \cellcolor[HTML]{FFF2E9}$54.61_{\pm 4.28}$ &
  \cellcolor[HTML]{FFDEC7}$74.98_{\pm 6.49}$ &
  \cellcolor[HTML]{FFDEC7}$67.31$ \\
\multirow{-2}{*}{\textbf{Ours}} &
  GLIP-Qwen &
  \cellcolor[HTML]{FFDEC7}$74.71_{\pm 3.52}$ &
  \cellcolor[HTML]{FFDEC7}$64.99_{\pm 1.29}$ &
  \cellcolor[HTML]{FFF2E9}$71.54_{\pm 5.61}$ &
  \cellcolor[HTML]{FFDEC7}$53.49_{\pm 2.15}$ &
  \cellcolor[HTML]{FFDEC7}$66.18$ &
  \cellcolor[HTML]{FFF2E9}$71.53_{\pm 4.87}$ &
  \cellcolor[HTML]{FFDEC7}$54.68_{\pm 1.91}$ &
  $73.49_{\pm 6.81}$ &
  \cellcolor[HTML]{FFF2E9}$66.57$ \\ \hline
\end{tabular}
\end{table*}

\begin{table*}[htbp!]
\centering
\fontsize{8.5}{13}\selectfont  
\setlength{\tabcolsep}{4pt} 
\caption{\textbf{Performance comparison under 3-shot settings  with F1 score (\%) reported.} 
The \colorbox[HTML]{FFDEC7}{\textbf{best}} and \colorbox[HTML]{FFF2E9}{second-best} results are highlighted. Results are averaged over five runs. }
\label{tb-3-F1}
\begin{tabular}{cc|ccccc|cccc}
\hline
\multicolumn{2}{c|}{} &
  \multicolumn{5}{c|}{\textbf{Feature Graphs}} &
  \multicolumn{4}{c}{\textbf{Text-Attributed Graphs}} \\
\multicolumn{2}{c|}{\multirow{-2}{*}{\textbf{\begin{tabular}[c]{@{}c@{}}F1 Score\\ 3-Shot\end{tabular}}}} &
  \textbf{MUTAG} &
  \textbf{PROTEINS} &
  \textbf{REDDIT} &
  \textbf{IMDB} &
  \textbf{Average} &
  \textbf{BBBP} &
  \textbf{BACE} &
  \textbf{E-Commerce} &
  \textbf{Average} \\ \hline
 &
  GraphMAE &
  $60.60_{\pm 9.93}$ &
  $39.23_{\pm 3.76}$ &
  $58.96_{\pm 8.04}$ &
  $44.31_{\pm 7.60}$ &
  $50.78$ &
  $44.65_{\pm 2.40}$ &
  $35.29_{\pm 0.17}$ &
  \cellcolor[HTML]{FFF2E9}$73.80_{\pm 3.19}$ &
  $51.25$ \\
 &
  GraphCL &
  $71.77_{\pm 2.78}$ &
  $58.73_{\pm 3.80}$ &
  $70.99_{\pm 3.40}$ &
  $46.65_{\pm 6.93}$ &
  $62.04$ &
  $54.18_{\pm 6.68}$ &
  $53.53_{\pm 3.24}$ &
  $68.53_{\pm 9.31}$ &
  $58.75$ \\
 &
  DGI &
  $66.43_{\pm 9.84}$ &
  $52.19_{\pm 8.60}$ &
  $67.19_{\pm 7.55}$ &
  $42.14_{\pm 4.05}$ &
  $56.99$ &
  $44.43_{\pm 2.24}$ &
  $40.44_{\pm 7.58}$ &
  $71.96_{\pm 6.05}$ &
  $52.28$ \\
\multirow{-4}{*}{\textbf{\begin{tabular}[c]{@{}c@{}}Graph\\ Pretrain\\ Method\end{tabular}}} &
  SimGRACE &
  $69.04_{\pm 8.35}$ &
  $52.39_{\pm 9.04}$ &
  $70.82_{\pm 5.68}$ &
  $39.28_{\pm 6.25}$ &
  $57.88$ &
  $51.20_{\pm 7.21}$ &
  $35.20_{\pm 2.17}$ &
  $72.89_{\pm 9.15}$ &
  $53.10$ \\ \hline
 &
  GraphGPT-Mistral &
  - &
  - &
  - &
  - &
  - &
  $50.30_{\pm 6.01}$ &
  $43.68_{\pm 9.31}$ &
  $68.53_{\pm 9.31}$ &
  $54.17$ \\
 &
  GraphGPT-Qwen &
  - &
  - &
  - &
  - &
  - &
  $51.31_{\pm 2.40}$ &
  $42.52_{\pm 8.12}$ &
  $69.31_{\pm 6.28}$ &
  $54.38$ \\
 &
  LLaGA-Mistral &
  $41.71_{\pm 9.81}$ &
  $39.16_{\pm 2.45}$ &
  $63.49_{\pm 5.62}$ &
  $38.55_{\pm 10.26}$ &
  $45.73$ &
  $46.40_{\pm 5.35}$ &
  $35.22_{\pm 3.63}$ &
  $72.90_{\pm 9.14}$ &
  $51.51$ \\
 &
  LLaGA-Qwen &
  $52.87_{\pm 9.49}$ &
  $55.97_{\pm 2.96}$ &
  $65.92_{\pm 6.90}$ &
  $43.04_{\pm 8.73}$ &
  $54.45$ &
  $46.76_{\pm 7.58}$ &
  $35.78_{\pm 1.29}$ &
  $64.04_{\pm 20.45}$ &
  $48.86$ \\
 &
  TEA-GLM-Mistral &
  $69.35_{\pm 1.53}$ &
  $57.83_{\pm 8.26}$ &
  $70.19_{\pm 4.81}$ &
  $53.04_{\pm 3.45}$ &
  $62.60$ &
  $52.66_{\pm 7.21}$ &
  $48.82_{\pm 9.68}$ &
  $69.23_{\pm 8.20}$ &
  $56.90$ \\
\multirow{-6}{*}{\textbf{\begin{tabular}[c]{@{}c@{}}Graph-\\ LLM\\ Method\end{tabular}}} &
  TEA-GLM-Qwen &
  $58.44_{\pm 9.27}$ &
  $58.25_{\pm 6.45}$ &
  $67.53_{\pm 8.72}$ &
  $51.25_{\pm 7.93}$ &
  $58.87$ &
  $49.31_{\pm 7.22}$ &
  $49.17_{\pm 6.66}$ &
  $66.75_{\pm 8.43}$ &
  $55.08$ \\ \hline
 &
  GLIP-Mistral &
  \cellcolor[HTML]{FFF2E9}$71.33_{\pm 6.35}$ &
  \cellcolor[HTML]{FFF2E9}$59.22_{\pm 6.06}$ &
  \cellcolor[HTML]{FFDEC7}$72.18_{\pm 2.27}$ &
  \cellcolor[HTML]{FFF2E9}$53.09_{\pm 4.51}$ &
  \cellcolor[HTML]{FFF2E9}$63.96$ &
  \cellcolor[HTML]{FFDEC7}$55.26_{\pm 7.86}$ &
  \cellcolor[HTML]{FFF2E9}$54.18_{\pm 4.43}$ &
  \cellcolor[HTML]{FFDEC7}$74.68_{\pm 6.62}$ &
  \cellcolor[HTML]{FFDEC7}$61.37$ \\
\multirow{-2}{*}{\textbf{Ours}} &
  GLIP-Qwen &
  \cellcolor[HTML]{FFDEC7}$72.75_{\pm 3.49}$ &
  \cellcolor[HTML]{FFDEC7}$60.90_{\pm 3.15}$ &
  \cellcolor[HTML]{FFF2E9}$71.03_{\pm 5.81}$ &
  \cellcolor[HTML]{FFDEC7}$52.79_{\pm 4.83}$ &
  \cellcolor[HTML]{FFDEC7}$64.37$ &
  \cellcolor[HTML]{FFF2E9}$54.02_{\pm 4.79}$ &
  \cellcolor[HTML]{FFDEC7}$50.43_{\pm 5.01}$ &
  $73.28_{\pm 6.93}$ &
  \cellcolor[HTML]{FFF2E9}$59.24$ \\ \hline
\end{tabular}
\end{table*}

\end{document}

%% file: symbol_def.tex
\newtheorem{problem}{Problem}
\newtheorem{theorem}{Theorem}[section]

\newenvironment{fminipage}%
{\begin{Sbox}\begin{minipage}}%
		{\end{minipage}\end{Sbox}\fbox{\TheSbox}}

\newcommand{\removelatexerror}{\let\@latex@error\@gobble}

\newcommand\LL{\bm{\mathit{L}}}

\newcommand\XX{\boldsymbol{\mathit{X}}}

\newcommand\zz{\boldsymbol{\mathit{z}}}

\newcommand\xx{\boldsymbol{\mathit{x}}}

\newcommand\uu{\boldsymbol{\mathit{u}}}

\renewcommand\AA{\boldsymbol{\mathit{A}}}
\newcommand\BB{\boldsymbol{\mathit{B}}}

\newcommand\JJ{\boldsymbol{\mathit{J}}}
\newcommand\DD{\boldsymbol{\mathit{D}}}

\newcommand\ZZ{\boldsymbol{\mathit{Z}}}

\newcommand\II{\boldsymbol{\mathit{I}}}

\DontPrintSemicolon
\SetKw{KwAnd}{and}
\SetFuncSty{textsc}
\SetKwInOut{Input}{Input\ \ \ \ }
\SetKwInOut{Output}{Output}